\PassOptionsToPackage{dvipsnames, table}{xcolor}

\documentclass{article}

\newif\ifpublic
\publictrue
\ifpublic
    \usepackage[main, preprint]{neurips_2026}
\else
    \usepackage{neurips_2026}
\fi
\usepackage{color,amsmath,amssymb, amsthm}
\usepackage{xspace}
\usepackage{enumitem}
\usepackage{hyperref}
\usepackage{overpic}
\usepackage{cleveref}
\usepackage{algorithm}
\usepackage{algpseudocode} 

\usepackage[utf8]{inputenc} % allow utf-8 input
\usepackage[T1]{fontenc}    % use 8-bit T1 fonts
\usepackage{hyperref}       % hyperlinks
\usepackage{url}            % simple URL typesetting
\usepackage{booktabs}       % professional-quality tables
\usepackage{amsfonts}       % blackboard math symbols
\usepackage{nicefrac}       % compact symbols for 1/2, etc.
\usepackage{microtype}      % microtypography
\usepackage{xcolor}         % colors
\usepackage{graphicx}       % includegraphics support
\usepackage{booktabs}
\usepackage{pifont}
\usepackage{adjustbox}
\usepackage{titletoc}
\usepackage{thmtools}
\usepackage{thm-restate}
\usepackage{longtable}
\usepackage{multirow}
\usepackage{colortbl}
\usepackage{array}
\usepackage{makecell}
\usepackage{wrapfig}
\usepackage[most]{tcolorbox}

\theoremstyle{plain}

\newtheorem{lemma}{Lemma}

\newtheorem*{theorem*}{Theorem}

\newif\ifshowcomments
\showcommentstrue % flip to \showcommentsfalse to hide
\ifshowcomments
  \newcommand{\gf}[1]{\textcolor{red}{GF: #1}}
  \newcommand{\ce}[1]{\textcolor{magenta}{CE: #1}}
  \newcommand{\mw}[1]{\textcolor{orange}{MW: #1}}
\else
  \newcommand{\gf}[1]{}
  \newcommand{\ce}[1]{}
  \newcommand{\mw}[1]{}
\fi

\newcommand{\inputprompts}{D}
\newcommand{\promptset}{\mathcal{D}}
\newcommand{\prunedset}{\mathcal{D}'}
\newcommand{\prompt}{S}
\newcommand{\decodetoken}{v_{i_k}}

\newcommand{\model}{\phi}
\renewcommand{\P}{\mathbb{P}}
\newcommand{\E}{\mathbb{E}}
\newcommand{\algo}{NightVision\xspace}
\newcommand{\hidden}{hidden dimension\xspace}
\newcommand{\params}{parameter count\xspace}
\newcommand{\depth}{depth\xspace}

\newcommand{\hd}{d}
\newcommand{\dep}{L}
\newcommand{\param}{P}
\newcommand{\dhat}{\hat{\hd}}
\newcommand{\Lhat}{\hat{\dep}}
\newcommand{\Phat}{\hat{\param}}
\newcommand{\dtrue}{d^{\star}}

\newcommand{\vocab}{\mathcal V}
\newcommand{\vocabsize}{|\vocab|}
\newcommand{\commonset}{\mathcal C}

\newcommand{\tokenset}{\mathcal T}

\newcommand{\zqi}{z_q^{(i)}}
\newcommand{\pqi}{p_q^{(i)}}
\newcommand{\yqi}{y_q^{(i)}}

\newcommand{\Mlp}{\mathbf Y}
\newcommand{\Mlpfull}{\mathbf Y_{\mathrm{full}}}
\newcommand{\Mlpstar}{\mathbf Y^{\star}}
\newcommand{\Mcarlini}{\mathbf Q}
\newcommand{\Mgram}{\mathbf G}

\newcommand{\pvec}{\mathbf p}
\newcommand{\zvec}{\mathbf z}

\newcommand{\pvecS}{\mathbf p_{\prompt}}
\newcommand{\zvecS}{\mathbf z_{\prompt}}
\newcommand{\yvecS}{\mathbf y_{\prompt}}
\newcommand{\simplex}{\Delta^{\vocabsize-1}}

\newcommand{\dffn}{d_{\mathrm{ffn}}}
\newcommand{\dkv}{d_{\mathrm{kv}}}
\newcommand{\nheads}{n_{\mathrm{heads}}}
\newcommand{\nkv}{n_{\mathrm{kv}}}

\newcommand{\Etot}{E_{\mathrm{tot}}}
\newcommand{\Eact}{E_{\mathrm{act}}}
\newcommand{\ttft}{TTFT\xspace}
\newcommand{\gqa}{GQA\xspace}
\newcommand{\moe}{MoE\xspace}

\newcommand{\eigi}{\lambda_i}

\newcommand{\slope}{\tilde s_i}
\newcommand{\isad}{i_{\mathrm{sad}}}
\newcommand{\ipeak}{i_{\mathrm{peak}}}
\newcommand{\aoi}{AOI\xspace}

\newcommand{\Tij}{T_{i,j}}
\newcommand{\Tref}{T_{i,j}^{\mathrm{ref}}}
\newcommand{\Ttgt}{T_{i,j}^{\mathrm{tgt}}}
\newcommand{\betaref}{\beta^{\mathrm{ref}}}
\newcommand{\alpharef}{\alpha^{\mathrm{ref}}}

\newcommand{\gammatgt}{\gamma^{\mathrm{tgt}}}
\newcommand{\zetaref}{\zeta^{\mathrm{ref}}}
\newcommand{\zetatgt}{\zeta^{\mathrm{tgt}}}
\newcommand{\transferk}{k}
\newcommand{\refset}{\mathcal F}
\newcommand{\refsetref}{\mathcal F^{\mathrm{ref}}}
\newcommand{\refsettgt}{\mathcal F^{\mathrm{tgt}}}

\title{Black-Box Inference of LLM Architectural Properties with Restrictive API Access}

\author{%
  Christopher Ellis\thanks{Equal contribution} \\
  Carnegie Mellon University\\
  Pittsburgh, PA 15213 \\
  \texttt{crellis@andrew.cmu.edu} \\
  \And
  Shreyas Chaudhari$^*$ \\
  Carnegie Mellon University\\
  Pittsburgh, PA 15213 \\
  \texttt{shreyasc@andrew.cmu.edu} \\
  \AND
  Mei-Yu Wang$^*$ \\
  Pittsburgh Supercomputing Center\\
  Pittsburgh, PA 15213 \\
  \texttt{mwang7@psc.edu} \\
  \And
  Leighton Barnes \\
  Carnegie Mellon University\\
  Pittsburgh, PA 15213 \\
  \texttt{leightonb@cmu.edu} \\
  \And
  Giulia Fanti \\
  Carnegie Mellon University\\
  Pittsburgh, PA 15213 \\
  \texttt{gfanti@andrew.cmu.edu} \\
  \And
  Jos\'{e} M. F. Moura \\
  Carnegie Mellon University\\
  Pittsburgh, PA 15213 \\
  \texttt{moura@andrew.cmu.edu} \\
}

\begin{document}

\maketitle

\begin{abstract}
In practice, most commercial LLM providers do not publicly release details of underlying LLM architectures. 
However, prior work has shown that given limited API access to an LLM (namely, top-$k$ logits and/or a logit bias function), one can recover certain architectural details of an LLM, such as the \hidden of the feed-forward network.
Perhaps in response to these results, most commercial LLM providers have restricted their APIs to expose only the single logit for each decoded token, and they no longer give users the ability to bias logits. 
We show that even under current restrictive APIs, several architectural parameters are still recoverable. 
We present \algo, an attack that uses restrictive black-box API access to estimate the \hidden, \depth, and \params of an LLM.
Algorithmically, \algo relies on a novel \emph{common set} prompting technique in which multiple prompts expose log probabilities for the same set of output tokens;
a spectral analysis of these results is used to infer \hidden.
\algo additionally uses end-to-end time to first token (\ttft) measurements and the estimated \hidden to estimate \depth and \params.
We empirically evaluate \algo on 32 open-source LLMs, recovering \hidden to within $23\%$ average relative error across all models ($9\%$ on \moe models), and \depth and \params to within $53\%$ for models exceeding
three billion parameters.
We run extensive ablations to demonstrate how these accuracies scale with token budget and model properties. 
Overall, our results suggest that current LLM APIs are not sufficiently restricted to fully obfuscate the architectural details of their underlying models. 
\end{abstract}

\section{Introduction}
\label{sec:intro}
Can attackers infer architectural properties of a large language model, e.g., its hidden dimension, depth, or total parameter count, from API access alone? This question is of great interest to both LLM developers and servers, who may seek to keep their hyperparameters and associated model intellectual property secret \citep{tramer2016stealing}. It is also relevant to security researchers and digital forensics experts, who may need to characterize unknown LLM APIs encountered in the wild. The feasibility of black-box architectural inference thus has direct implications for both API design \citep{hartenstein2025bridging, liu2026exploring} and model auditing \citep{amirizaniani2024auditllm, mokander2024auditing, rastogi2023supporting}.

Recent black-box attacks \citep{zhao2025survey} have shown that hidden dimension and final-layer weights (up to orthogonal transformation) can be recovered via spectral analysis of output logits \citep{carlini2024stealing, finlayson2024logits}. These attacks critically depend on either \emph{full log probabilities} or a \emph{logit-bias} parameter that allows the attacker to reconstruct full logit vectors from top-$k$ probabilities. In response, as modern language model APIs continue to evolve \citep{chauvin2025log}, providers have started to restrict access to both top-$k$ logits and logit bias \citep{carlini2024stealing}, rendering these attacks inapplicable on such APIs. A natural question is whether architectural inference remains feasible under a more restrictive API that models those of current frontier models. We answer affirmatively. Assuming only the log probability of a \emph{single} decoded token at each step (the minimal probabilistic information consistent with modern LLM APIs) together with end-to-end response timing, which cannot easily be hidden, we recover not only the hidden dimension but also depth and total parameter count, within a margin of error.

\textbf{Contributions.} Our contributions are threefold and summarized below. Overall, our results demonstrate that meaningful inference of sensitive architectural properties remains possible even under restrictive API access, and that closing this channel will require defending against timing side channels in addition to logit access.

\noindent
\emph{(1) Algorithmic:} We introduce \algo, a two-pronged inference algorithm that uses black-box API access to jointly recover the \hidden, \depth, and \params of transformer-based LLM architectures. 

The first prong of \algo introduces \emph{common-set prompting}, a novel method that recovers the hidden dimension under single-logit API access, requiring neither logit bias nor top-$k$ access. Second, we develop a \emph{timing-based recovery} procedure for depth and parameter count that exploits the distinctive scaling of prefill-dominated transformer inference with depth, sequence length, and hidden dimension; this stage uses our hidden-dimension estimate as a building block. 

\emph{(2) Theoretical:} We analytically bound in Section \ref{sec:theory} the number of API calls needed to successfully recover the requisite ``common set'' used in \algo. Via a
concentration argument on the size of the intersection of randomly-sampled
token sets across multiple prompts, we derive a sample-complexity result
that quantifies how many samples per prompt are required to form a
common set of a target size, or equivalently, the largest hidden
dimension recoverable with high probability under a fixed sample budget.

\emph{(3) Experimental:}
We evaluate \algo on 32 open-source LLMs, showing that it recovers the hidden dimension within $23\%$ mean relative error (exact recovery in 4 of 32 cases, and within $10\%$ in 12 of 32). Conditioned on an accurate hidden dimension, \algo recovers depth and total parameter count within ${\sim}53\%$ mean relative error on $\geq 3$B-parameter models.

The accuracy of our estimates depends in part on the \emph{budget} of tokens given to \algo. 
We also provide a fine-grained characterization of how estimation accuracy scales with budget and architecture properties. 
\section{Technical Preliminaries}
\label{sec:prelim}
\paragraph{Architectural Parameters of LLMs.}
Most LLMs can be modeled as stacks of $L$ transformer decoder layers
\citep{vaswani2017attention}. 
In these architectures, each transformer block operates on hidden representations of
a fixed dimension $d$. We refer to $L$ as the \emph{depth}, $d$ as the
\emph{hidden dimension}, and $P$ as the \emph{parameter count} totaled over the entire architecture; together
they determine inference cost, memory footprint, and capability, and are
commonly held proprietary by LLM providers. The hidden dimension $d$
admits a particularly clean recovery target: the next-token logits are
produced as $\zvec = \mathbf{W}\mathbf{h}$, where
$\mathbf{h} \in \mathbb{R}^d$ is the last-layer hidden state and
$\mathbf{W} \in \mathbb{R}^{\vocabsize \times d}$ is the output
embedding matrix over a vocabulary $\vocab$. Since
$\vocabsize \gg d$ in practice, the matrix of logits collected across
many prompts has rank at most $d$. This \emph{softmax bottleneck} is the
structural fact exploited by prior attacks
\citep{carlini2024stealing,finlayson2024logits} and by our spectral
method in \Cref{sec:spectral}.

\paragraph{Threat Model.}
Consider a \emph{target model} $\phi$ that consists of stacked transformer blocks, each with the same \hidden $d$ and \depth $L$, as is common in modern LLM architectures. Given API access to $\phi$, our goal is to estimate $d$, $L$, and \params $\param$.
Given an input sequence $\prompt = (x_1, \dots, x_{t-1})$, an LLM produces a
next-token logit vector $\zvecS \in \mathbb{R}^{\vocabsize}$ with entries
$z_{\prompt}^{(i)}$. Many APIs expose a temperature parameter $\tau > 0$ that rescales
the logits prior to the softmax, yielding a probability vector
$\pvecS(\tau) = \mathrm{softmax}(\zvecS/\tau) \in \simplex$
with entries
$p_{\prompt}^{(i)}(\tau)\!=\!\exp(z_{\prompt}^{(i)}/\tau) / \sum_{j} \exp(z_{\prompt}^{(j)}/\tau)
\! =\! \P(x_t = v_i \mid x_1, \dots, x_{t-1})$
and a corresponding log-probability vector $\yvecS(\tau)$ with entries
$y_{\prompt}^{(i)}(\tau)\!=\!\log p_{\prompt}^{(i)}(\tau)$;
larger $\tau$ flattens the distribution toward uniform, and $\tau = 1$ recovers
the unscaled distribution $\pvecS \triangleq \pvecS(1)$. We use $v_i$ for $i\in[\vocabsize]$ to denote the $i$th token label. We operate in a restricted
black-box setting in which none of $\zvecS$, $\zvecS(\tau)$, $\pvecS$, or $\pvecS(\tau)$
is observed. Once the model decodes a next token $v_{i_0}$ by sampling under
$\pvecS(\tau)$, only the realized log-probability
$y_{\prompt}^{(i_0)}(\tau)$ is revealed. When prompts are indexed as $\prompt_q$ (as in our common-set construction in \Cref{sec:common_set}), we abbreviate $\zqi \triangleq z_{\prompt_q}^{(i)}$, $\pqi \triangleq p_{\prompt_q}^{(i)}$, and $\yqi \triangleq y_{\prompt_q}^{(i)}$. Hence, our threat models are strictly weaker (give less information to the attacker) than those of \citet{carlini2024stealing} and \citet{finlayson2024logits},
which additionally grant top-$k$ entries of $\pvecS(\tau)$ or
logit-bias control prior to the final softmax. Our common-set
prompting method (\Cref{sec:common_set}) benefits from flatter
distributions, so we use $\tau = 2$ (the cap on most
current APIs)

\section{Related Work}
\label{sec:related}
\citet{carlini2024stealing} and \citet{finlayson2024logits} concurrently showed that an LLM's hidden dimension can be recovered from API access by exploiting the softmax bottleneck: the matrix of output logits across many prompts has rank equal to the hidden dimension, which can be readily determined via singular value decomposition. \citet{carlini2024stealing} further recover the final output layer up to orthogonal transformations. Crucially, both methods depend on a logit-bias parameter to reconstruct full logit vectors from the API’s top-$k$ log-probability outputs (with $k \geq 1$) \citep{morris2023language}, and both are rendered ineffective once this functionality is removed. We provide a detailed algorithmic comparison between our method and that of \citet{carlini2024stealing} in \Cref{sec:common_set}.

\begin{table}[htpb]
	\centering
	\footnotesize
	\caption{Prior work on black-box recovery of architectural properties from production LLM APIs. The left block shows the access the attack \emph{requires}; the right block shows what it \emph{recovers}. All prior attacks that recover hidden dimension require a logit-bias parameter (used to \emph{reconstruct} full logit vectors from a single decoded log-probability, which is why ``Full logits'' is not checked), a feature that has been removed or restricted by major providers in response.}
	\label{tab:prior_work}
	\setlength{\tabcolsep}{5pt}
	\begin{tabular}{@{}lccccc|ccc@{}}
		\toprule
		& \multicolumn{5}{c|}{\textbf{Required access}} & \multicolumn{3}{c}{\textbf{Recovered}} \\
		\cmidrule(lr){2-6} \cmidrule(lr){7-9}
		\textbf{Work}
		& \makebox[2.2em]{\shortstack{\scriptsize Text\\\scriptsize only}}
		& \makebox[2.2em]{\shortstack{\scriptsize Full\\\scriptsize logits}}
		& \makebox[2.2em]{\shortstack{\scriptsize Logit\\\scriptsize bias}}
		& \makebox[2.2em]{\shortstack{\scriptsize Single \\\scriptsize logprob}}
		& \makebox[2.6em]{\shortstack{\scriptsize Inference\\\scriptsize timing}}
		& Hidden & Depth & \#Params \\
		\midrule
		\citet{carlini2024stealing}      &            &            & \checkmark & \checkmark &            &    \checkmark        &  &            \\
		\citet{finlayson2024logits}      &            &            & \checkmark & \checkmark &            &    \checkmark        &  &            \\
		\citet{li2026ikp}                & \checkmark &            &            &            &            &            &            & \checkmark \\
		\textbf{\algo (basic common-set)}               &            &            &            & \checkmark &  & \checkmark & &  \\
		\textbf{\algo (with timing)}                &            &            &            & \checkmark & \checkmark & \checkmark & \checkmark & \checkmark \\
		\bottomrule
	\end{tabular}
\end{table}

More recently, \citet{li2026ikp} take an alternative approach, in which rather than probing the model's internals, they attempt to probe what it has learned. Their method, \emph{Incompressible Knowledge Probes} (IKPs), tests a model on obscure factual questions whose answers cannot be inferred from general reasoning and must instead be memorized in the model's weights. Since storing more facts requires more parameters, a model's accuracy on these probes serves as a proxy for its size, once calibrated against open-weight models of known size. Table~\ref{tab:prior_work} positions our work against these methods by threat model, namely, the API access each attack requires, and which architectural properties it recovers. Our method requires the strictly weakest probabilistic API access of any approach that recovers the hidden dimension, and it is the only method that jointly recovers depth, hidden dimension, and parameter count.
Due to space constraints, we discuss additional related work in Appendix \ref{app:related}.

\section{\algo: Algorithm and Analysis}
\label{sec:methods}
\begin{wrapfigure}[8]{r}{0.44\textwidth}
    \centering
	\vspace{-0.5in}
	\includegraphics[width=0.9\linewidth]{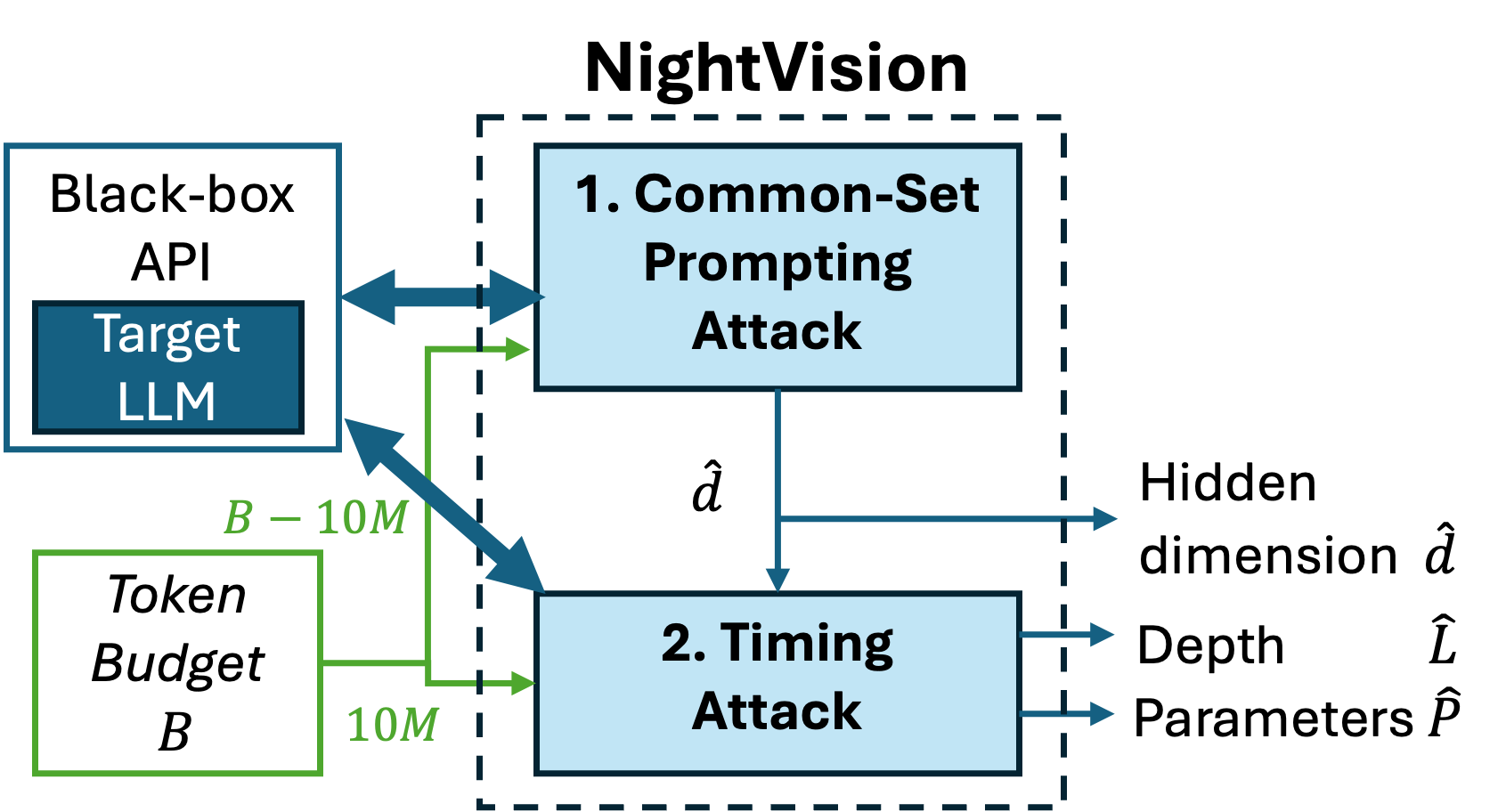}
	\caption{\footnotesize \algo overview}
	\label{fig:overview}
\end{wrapfigure}

We first give an overview of \algo, as illustrated in Figure \ref{fig:overview}, followed by a detailed explanation of its components.
\algo proceeds in two phases. The first phase, which we call the  \emph{common-set prompting attack}, is used to 
recover the \hidden $d$ (\Cref{sec:common_set}). 
The second phase, which we refer to as the \emph{timing attack} recovers the \params and \depth, given an estimate for \hidden $\dhat$ obtained in phase 1 (\Cref{sec:timing}).
Given a total budget of $B$ tokens for the attack (we find that $B > 2 \times 10^{11}$ tokens are required in practice for LLMs with billions of parameters), we allocate the first $10$ million tokens to the timing attack, and the rest to the common-set prompting attack.

\subsection{Phase 1: \algo Common-Set Prompting Attack}
\label{sec:common_set}

We first review the attack of \citet{carlini2024stealing} under top-1 logit access and explain why it is infeasible under modern APIs. We then present our common-set prompting attack (\Cref{sec:common-set}) and analyze its sample complexity (\Cref{sec:theory}). 

\definecolor{stagebg}{RGB}{235, 235, 240}
\definecolor{stagefg}{RGB}{60, 60, 90}
\definecolor{carlinihl}{RGB}{178, 34, 52}
\definecolor{ourshl}{RGB}{29, 96, 158}
\definecolor{boxrule}{RGB}{180, 180, 190}
\newcommand{\hlC}[1]{\textcolor{carlinihl}{\textbf{#1}}}
\newcommand{\hlO}[1]{\textcolor{ourshl}{\textbf{#1}}}

\newcommand{\stagebanner}[1]{%
	\par\vspace{0.1pt}%
	\noindent\colorbox{stagebg}{%
		\makebox[\linewidth][c]{%
			\rule[-0.05em]{0pt}{0.1em}\textcolor{stagefg}{\scriptsize\itshape #1}%
		}%
	}%
	\par\vspace{0.1pt}%
}

\begin{figure}[htpb]
	\centering
	\begin{tcolorbox}[
		colback=white, colframe=white, boxrule=0.6pt, arc=2pt,
		left=2pt, right=2pt, top=0pt, bottom=2pt,
		boxsep=0pt,
		]
		
		\noindent
		\begin{minipage}[t]{0.495\linewidth}
			\colorbox{carlinihl}{\makebox[\linewidth][c]{%
					\rule[-0.05em]{0pt}{0.1em}\color{white}\bfseries\scriptsize \citet{carlini2024stealing}: Top-$1$ binary logit-bias attack}}
		\end{minipage}\hfill
		\begin{minipage}[t]{0.495\linewidth}
			\colorbox{ourshl}{\makebox[\linewidth][c]{%
					\rule[-0.05em]{0pt}{0.1em}\color{white}\bfseries\scriptsize \algo  Phase 1: Common-set prompting attack}}
		\end{minipage}

		\par\vspace{1pt}
		\noindent
		\begin{minipage}[t]{0.495\linewidth}
			\begin{algorithmic}[1]
				\scriptsize
				\Require \\
				Top-$1$ logprob API \hlC{with binary bias $b \in \{-1, 0\}$}\\
				Prompts $\promptset=\{\prompt_q\}_{q=1}^{n}$ with $n > d$
			\end{algorithmic}
		\end{minipage}\hfill
		\begin{minipage}[t]{0.495\linewidth}
			\begin{algorithmic}[1]
				\scriptsize
				\Require \\
				Single sampled logprob API \hlO{(no logit bias)}\\
				\hlO{High-entropy} prompts $\promptset=\{\prompt_q\}_{q=1}^{\inputprompts}$, number of samples $N$, temperature $\tau$
			\end{algorithmic}
		\end{minipage}

		\stagebanner{Stage 1: Collect Outputs}
		
		\noindent
		\begin{minipage}[t]{0.495\linewidth}
			\begin{algorithmic}[1]
				\scriptsize
				\setcounter{ALG@line}{2}
				\For{$q = 1, \dots, n$}
				\State \hlC{Query unbiased}; record top token $r$ and its logprob $y_{\mathrm{top}}$
				\For{{each $t \in \vocab \setminus \{r\}$}}
				\State \hlC{Query biased}; record logprob $y'_{\mathrm{top}}$
				\State \hlC{Recover prob.}: $p_q^{(t)} = \dfrac{\exp(y_{\mathrm{top}} - y'_{\mathrm{top}}) - 1}{1/e - 1}$
				\State \hlC{Convert to logit}: $z_q^{(t)} = \log p_q^{(t)} - y_{\mathrm{top}}$
				\EndFor
				\EndFor
			\end{algorithmic}
		\end{minipage}\hfill
		\begin{minipage}[t]{0.495\linewidth}
			\begin{algorithmic}[1]
				\scriptsize
				\setcounter{ALG@line}{2}
				\For{$q = 1, \dots, \inputprompts$}
				\For{\hlO{$k = 1, \dots, N$}}
				\State \hlO{Sample at temperature \ $\tau$}; record $\left(\decodetoken,\, \yqi\right)$
				\EndFor
				\State {$\tokenset_q \gets \{\text{unique tokens observed}\}$}
				\EndFor
				
			\end{algorithmic}
		\end{minipage}
		
		\stagebanner{Stage 2: Form Matrix}
		
		\noindent
		\begin{minipage}[t]{0.495\linewidth}
			\begin{algorithmic}[1]
				\scriptsize
				\setcounter{ALG@line}{8}
				\State \hlC{Build $\Mcarlini \in \mathbb{R}^{n \times |\vocab|}$} with $\Mcarlini_{q,t} = z_q^{(t)}$
			\end{algorithmic}
		\end{minipage}\hfill
		\begin{minipage}[t]{0.495\linewidth}
			\begin{algorithmic}[1]
				\scriptsize
				\setcounter{ALG@line}{8}
				\State \hlO{Common set $\commonset$  and common prompts $\prunedset$ formed via greedy pruning of sparse $\inputprompts{\times}|\vocab|$ matrix (\Cref{sec:ctm})}
				\State \hlO{Build $\Mlp \in \mathbb{R}^{D' \times |\commonset|}$} with $Y_{q,i} = y_q^{(v_i)}$ for $v_i \in \commonset$
			\end{algorithmic}
		\end{minipage}
		
		\stagebanner{Stage 3: Estimate $d$ via spectrum}
		
		\noindent
		\begin{minipage}[t]{0.495\linewidth}
			\begin{algorithmic}[1]
				\scriptsize
				\setcounter{ALG@line}{10}
				\State \hlC{Compute SVD of ${\Mcarlini}$}; extract $\sigma_1 \geq \sigma_2 \geq \cdots$
				\State \Return {$\dhat$ at the largest gap in $\{\sigma_i\}$}
			\end{algorithmic}
		\end{minipage}\hfill
		\begin{minipage}[t]{0.495\linewidth}
			\begin{algorithmic}[1]
				\scriptsize
				\setcounter{ALG@line}{10}
				\State \hlO{Eigendecompose $\Mgram = \Mlp\Mlp^{\top}$}; extract $\lambda_1 \geq \lambda_2 \geq \cdots$
				\State \Return \hlO{$\dhat$ via saddle-then-peak inflection on $a_i = \arctan(\eigi)$ (saddle at $a_i\!\le\!\arctan(\pi)$, peak at $a_i\!\le\!\arctan(\pi/2)$; \Cref{sec:spectral})}
			\end{algorithmic}
		\end{minipage}
	\end{tcolorbox}
	\caption{Comparison of our common-set prompting attack to Top-1 binary logit-bias attack of \citet{carlini2024stealing}.}
	\label{fig:algo_comparison}
	\vspace{-0.1in}
\end{figure}

\paragraph{The attack of \citet{carlini2024stealing}}
\label{sec:carlini-challenges}

The attack of \citet{carlini2024stealing} follows a three-stage procedure, illustrated in Figure~\ref{fig:algo_comparison} (left): (1) collect outputs from prompts, (2) assemble a logit matrix, and (3) estimate $\dhat$ from its singular values.
In stage~(1), they query the target LLM $\phi$ with a set $\promptset$ of $n>d$ distinct prompts.
For each prompt $q$, they record the top token $r$ and its log-probability $y_{\mathrm{top}}$. Then, for each non-top token $t$, they submit a query with a logit bias of $-1$ on $t$, which suppresses $t$ and yields a new top log-probability $y'_{\mathrm{top}}$ for $r$. The shift $y_{\mathrm{top}} - y'_{\mathrm{top}}$ depends on the original probability of $t$, and inverting it in closed form gives $p_q^{(t)} = (\exp(y_{\mathrm{top}} - y'_{\mathrm{top}}) - 1)/(1/e - 1)$, the probability of $t$ under prompt $q$.
In stage~(2), these probabilities are converted to logits and collected into an $n \times |\vocab|$ matrix $\Mcarlini$, with one row per prompt and one column per token in $\vocab$.
In stage~(3), the singular values of $\Mcarlini$ recover the hidden dimension: the rank shows up as a pronounced drop in the ordered singular values (\Cref{app:changepoint}), at an index that typically matches the true $d$.
This attack is infeasible in our threat model, where logit bias is unavailable: without it, individual token probabilities cannot be recovered and $\Mcarlini$ cannot be constructed.

\subsubsection{Common-Set Prompting}
\label{sec:common-set}
Our common-set prompting attack follows the same three high-level stages as \citet{carlini2024stealing}: (1) output collection, (2) matrix formation, and (3) spectral estimation. Each stage is instantiated differently, as shown in \Cref{fig:algo_comparison}.

\paragraph{(1) Output Collection: Common-Set Prompting}
Our goal is to construct a matrix $\Mlp$ analogous to $\Mcarlini$ from \citet{carlini2024stealing}, in which each row corresponds to a prompt from a set $\prunedset$, each column corresponds to a token from a set $\commonset \subseteq \vocab$, and entries are log-probabilities of tokens $c\in \commonset$ under prompts $\prompt\in \prunedset$.
Without logit bias, a token's log-probability is only visible when the model decodes it. We pick $\prunedset$ and $\commonset$ jointly so that, under repeated sampling of $\phi$, \emph{every} token in $\commonset$ appears at least once for \emph{every} prompt in $\prunedset$.
Spectral recovery requires $|\commonset|>d$, so we seek high-entropy input prompts likely to produce a large vocabulary of output tokens.

Specifically, we start with a set $\promptset=\{\prompt_q\}_{q=1}^{\inputprompts}$ of $\inputprompts$ prompts chosen to induce high-entropy distributions over the next token (details below). For each prompt $\prompt_q \in \promptset$, we draw $N$ independent samples from $\phi(\prompt_q)$ at temperature $\tau$ (concrete values for $\inputprompts$, $N$, and $\tau$ are listed in \Cref{app:hyperparams}).
Each call $k\in [N]$ returns a single decoded token $\decodetoken$ sampled from the
temperature-scaled distribution $\pvec(\tau)$, together with its log-probability $y_q^{(i_k)}$ under that distribution. No other entries of the output distribution are revealed. 

Let $\tokenset_q \subseteq \vocab$
denote the set of unique tokens decoded under prompt $\prompt_q$ across its $N$ samples, $\tokenset_q \triangleq \cup_{k\in [N]} \{\decodetoken\}$. We define the \emph{common set} as $\commonset \;\triangleq \; \cap_{q\in [\inputprompts]} \tokenset_{q}$, i.e., the set of tokens decoded at least once under every prompt.

The number of samples $N$ required to form a common set of a given size depends on the next-token distributions. 

Increasing temperature $\tau$ alone is typically insufficient to drive $|\commonset|$ above $d$, so prompt selection becomes the dominant lever for vocabulary coverage (\Cref{app:entropy}). We find that out-of-distribution token sequences (inputs the model has likely not seen during training) reliably produce diffuse predictions. 

Our automated prompt search (\Cref{app:entropy}) draws candidates from character groups including Latin, CJK, emoji, Arabic, Hebrew, Braille, and mathematical symbols.

\paragraph{(2) Sampling and Common Token Matrix Construction}
\label{sec:ctm}

We populate a sparse log-probability matrix $\hat{\Mlpfull} \in \mathbb{R}^{D \times |\vocab|}$ with $[\hat{\Mlpfull}]_{q,i} = \yqi$ for $v_i \in \tokenset_q$ and $\mathrm{NaN}$ elsewhere. Different prompts cover different parts of the vocabulary, so $\hat{\Mlpfull}$ contains many missing entries and cannot be fed to the spectral step directly. We apply greedy NaN-pruning (\Cref{app:greedy}) to extract a dense submatrix: the surviving columns form the \emph{common set} $\commonset$, and the surviving rows define a subset of prompts $\prunedset \subseteq \promptset$ retained for spectral analysis. Let $D' = |\prunedset|$ denote the post-pruning prompt count; the resulting dense submatrix $\Mlp \in \mathbb{R}^{D' \times |\commonset|}$ feeds the spectral step.

\paragraph{(3) Spectral Hidden-Dimension Estimation}
\label{sec:spectral}

We recover $d$ from $\Mlp$ via spectral analysis, following \citet{carlini2024stealing}. The singular values of $\Mlp$ are noisy, so we work instead with the Gram matrix $\Mgram = \Mlp\Mlp^\top$ (or $\Mlp^\top\Mlp$ when $D' > |\commonset|$) to amplify the signal. Its eigendecomposition $\Mgram = \mathbf{U}\,\boldsymbol{\Lambda}\,\mathbf{U}^\top$ yields ordered eigenvalues $\lambda_1 \geq \lambda_2 \geq \cdots \geq \lambda_{\min(D',|\commonset|)} \geq 0$.
The leading eigenvalues span several orders of magnitude while the tail is comparatively flat, so the elbow rule of \citet{carlini2024stealing} does not isolate a clean inflection. We use a custom change-point detector (\Cref{app:changepoint}): a transform of the spectrum exposes the head-to-tail transition, two geometric landmarks bracket $\dhat$ from above and below, and the point of maximum curvature within that bracket is taken as $\dhat$. The detector exposes a few hyperparameters (slope half-window, smoothing window, prominence fraction); the values used throughout the evaluation are listed in \Cref{app:hyperparams}.
Finally, we snap $\dhat$ to a grid spacing $g = 128$, since hidden dimensions of deployed models are almost always multiples of powers of two.

\subsubsection{Theoretical Analysis: Sample Complexity of the Common-Set Attack}
\label{sec:theory}
We bound the number of samples per prompt sufficient to construct a common-set matrix large enough for spectral recovery of the hidden
dimension. The spectral method (\Cref{sec:spectral}) operates on a matrix
$\Mlp \in \mathbb{R}^{D' \times |\commonset|}$ whose rank reflects the hidden dimension $d$, so recovery requires both $D' \geq d$ and $|\commonset| \geq d$. The first condition is controlled by the
attacker via the size $D \ge D'$ of the prompt collection; the second is a probabilistic event determined by how many samples are drawn per prompt. We hence study the sample size $N$ required to ensure $|\commonset| \geq d$ with high probability. 

\begin{restatable}{theorem}{SampleComplexity}[Sample complexity for spectral recovery]
	\label{thm:uniform-sample-complexity}
	Let $\vocab$ be a finite vocabulary of size $\vocabsize$, and let
	$\tokenset_1, \dots, \tokenset_D$ be $D$ independent random subsets of $\vocab$, each
	formed by drawing $N$ tokens uniformly at random with replacement from
	$\vocab$. Let $\commonset \triangleq \tokenset_1 \cap \cdots \cap \tokenset_D$ denote their common intersection, and $d \leq D$ be a target threshold. If
	\begin{equation}
		N \geq \frac{\ln\!\left[1 - \left(\frac{d + \sqrt{2d\ln(1/\delta)} + 2\ln(1/\delta)}{\vocabsize}\right)^{1/D}\right]}{\ln(1 - 1/\vocabsize)}
		\label{eq:uniform-sample-complexity-bound}
	\end{equation}
	then $|\commonset| \geq d$ with probability at least $1 - \delta$.
\end{restatable}
(Proof in \Cref{app:proofs})
Expanding
\eqref{eq:uniform-sample-complexity-bound} (also in \Cref{app:proofs}) shows that the total sufficient query budget roughly 
scales as $DN = O(\vocabsize D\ln D)$.
\Cref{thm:uniform-sample-complexity} assumes uniformity, an idealization
of the high-entropy prompts produced by our search
(\Cref{sec:common_set}). \Cref{cor:t-flat-sample-complexity} in
\Cref{app:proofs} generalizes the bound to non-uniform distributions over a sub-vocabulary $\vocab' \subseteq \vocab$. The same form holds with $\vocabsize$ and $1/\vocabsize$ in the numerator and denominator respectively replaced by $|\vocab'|$ and
$t$ in the appropriate places. We empirically validate the
predicted scaling against measured common-set sizes in
\Cref{sec:eval-common-set}.

\subsection{Phase 2: \algo Timing Attack}
\label{sec:timing}
In Phase 2, \algo uses end-to-end inference time, from API call to token return, as a complementary side-channel signal to infer \depth and \params.

The timing attack (full pseudocode in \Cref{fig:algo_timing}) has three sub-phases: (1) timing collection, (2) \depth estimation, and (3) \params inference.
Our attack assumes access to a collection of $K$ models $\mathcal{F} = \{\model_i\}_{i=1}^{K}$, which could be obtained from open-source models. These are used to learn a timing scaling relation, which is applied to our target model $\model$. 

\paragraph{(1) Timing Collection}
Transformer-based LLM inference comprises two stages: \textbf{prefill}, in which the model processes the input prompt to build hidden states and populate the per-layer key-value (KV) cache, and \textbf{decode}, in which tokens are generated autoregressively, reusing cached keys and values to avoid recomputing the full context. 
The prefill stage tends to be more predictable computationally due to optimizations once the KV cache is filled, such as Grouped-Query Attention (\gqa)~\citep{ainslie2023gqa}. Hence, we design prompts with long input sequences and minimal generation length so that end-to-end latency is dominated by prefill. 

To make prefill dominate, we create a set of $N'$ diverse input sequence lengths $\mathcal{L} = \{\ell_j\}_{j=1}^{N'}$ (distinct from the prompt set $\promptset$ in Section~\ref{sec:spectral}).
For each input length $\ell_j$ ($j \in [N']$) and each training model $\model_i$, we collect an end-to-end timing signal $T_{i,j}$.

\paragraph{(2) Depth $L$ Estimation}

To estimate $L$, we model the prefill-dominated runtime $T_{i,j}$
for model $\model_i$ on input length $\ell_j$ as: 
\begin{equation}
	T_{i,j} \approx \beta\, L_{i}\, \ell_{j}^{2}\, d_{i} + \alpha,
	\label{eq:scaling_relation}
\end{equation}
where $L_i$ is the layer count, $d_i$ the true hidden dimension, and $\beta$,
$\alpha$ are empirical fitting parameters. The $L_i\, \ell_j^2\, d_i$ term
reflects the per-layer $O(\ell^2 d)$ attention cost~\citep{vaswani2017attention, dao2022flashattention}, which dominates the $O(\ell\, d^2)$ projection and FFN terms in the long-context
regime~\citep{kaplan2020scaling}. 
We show that this is a good model for timing measurements empirically in \Cref{sec:eval-timing}.
The coefficient $\beta$ reflects a combination of system-level factors (kernel efficiency, e.g., FlashAttention~\citep{dao2022flashattention}; hardware bandwidth; numerical precision; parallelism) and architectural choices (multi-head versus
grouped-query attention (\gqa)~\citep{ainslie2023gqa}; two-matrix GELU versus three-matrix SwiGLU FFN). Despite these sources of variation, we find that a single global $\beta$ provides a good approximation for models exceeding approximately 3B parameters that employ LLaMA-style transformer blocks, the predominant architectural choice in modern large language models.

Our basic estimator for $L$ for a target model $\model$ learns a global $\beta$ and $\alpha$ from our training set $\mathcal F$ via linear regression, using the variety of prompt lengths in the prompting set. 
Next, we plug $\alpha$, $\beta$, and the \hidden estimate $\dhat$ from the common-set token attack into \eqref{eq:scaling_relation} to solve for  $L$. 
We intentionally retain a single linear term in the fitting relation and ignore system-level variations stemming from diverse inference engine configurations and accelerator types, which would otherwise confound multi-term fits. 
We discuss the effect of model architecture (including \moe models) in \Cref{app:architecture-accuracy} and extend this algorithm to a setting where the training set is run on different hardware than the target model in \Cref{app:cross-hardware}.

\paragraph{(3) Total Parameter Count $\param$ Estimation}
Given our estimated hidden dimension $\dhat$, estimated layer count $\Lhat$, and
an upper bound on vocabulary size $\hat{V}$, we estimate the total parameter count assuming a
LLaMA-style decoder-only transformer with \gqa,
SwiGLU feed-forward layers, and untied input/output embeddings. 
Although this architecture is quite different from many LLMs, we find that this model is sufficiently close in practice to predict parameter counts to reasonable accuracy (at least within an order of magnitude).
\begin{equation}
    \Phat = 2\hat{V}\dhat
    + \Lhat\left(2 + 2\,\frac{\nkv}{\nheads}
    + 3\,\frac{\dffn}{\dhat}\right)\dhat^{\,2}
    + (2\Lhat+1)\dhat,
    \label{eq:total_param}
\end{equation}
where $\nkv/\nheads$ is the key-value head ratio and
$\dffn/\dhat$ is the feed-forward expansion ratio. The first term
accounts for the input embedding and output projection matrices (untied, as in
LLaMA). The second term captures the per-layer parameter budget: query and
output projections contribute $2\dhat^2$, key-value projections contribute
$2(\nkv/\nheads)\dhat^2$, and the three SwiGLU weight
matrices ($W_{\mathrm{gate}}$, $W_{\mathrm{up}}$, $W_{\mathrm{down}}$)
contribute $3(\dffn/\dhat)\dhat^2$. The third term accounts for
RMSNorm parameters: two per layer (pre-attention and pre-FFN) plus one final
norm, each with $\dhat$ learnable scale parameters.\footnote{In practice this term is
negligible, e.g., for $\dhat=4096$ and $\Lhat=32$ it constitutes
${\sim}0.3$M out of ${\sim}8$B total parameters.
Architecture-specific ratios are retrieved from a per-model registry when
available; otherwise, we default to values representative of LLaMA-3 8B
($\nkv/\nheads = 1/4$,
$\dffn/\dhat = 3.5$).} 

\section{Results}
\label{sec:eval}

\Cref{sec:eval-overall} establishes overall efficacy; \Cref{sec:eval-common-set} isolates the factors governing hidden-dimension estimation via common-set prompting; \Cref{sec:eval-timing} examines parameter-count and depth estimation.

\paragraph{Models}
We evaluate \algo across 32 LLMs (full list in \Cref{app:models}). Our suite spans two architectural classes, \emph{dense} and \emph{mixture-of-experts (\moe)} transformers, drawn from the Qwen, IBM Granite, Llama, OLMo, and HuggingFace model families. Together they cover 135M--30B parameters, hidden dimensions $d \in [576, 4096]$, and depths $L \in [16, 48]$.

\paragraph{Test bench}
Log-probability matrices and timing measurements are collected with HuggingFace \texttt{transformers} and spot-checked against \href{https://github.com/vllm-project/vllm}{vLLM} (v0.16.0). For the timing measurements used in the scaling-relation analysis, we serve models via vLLM and measure end-to-end inference latency through its API.
The \algo hyperparameters used throughout this section (number of prompts, samples per prompt, temperature, change-point detector windows, grid spacing, and precision) are fixed across all models and listed in \Cref{app:hyperparams}; further details on our timing measurements and experimental setup are in \Cref{app:timing-experimental-setup}.

\paragraph{Metrics}
We report cost in two units, token consumption and number of LLM calls,
\[
\text{Token Consumption} = (\text{System Prompt Length} + \text{Input Tokens} + \text{Output Tokens}) \cdot (\#\,\text{Samples})
\]
that conservatively assumes no system-prompt caching. Plots showing error vs.\ token budget use the average system-prompt length across the 39 models, $\bar{\ell}_{\text{sys}}\approx 31$ tokens. We also report \emph{relative error}: for parameter $\theta$ with estimate $\hat\theta$, we define $R(\hat\theta;\theta) = |\theta - \hat\theta|/\theta$.

\subsection{Overall Results}
\label{sec:eval-overall}

We measure the accuracy of \algo in estimating \hidden, \params, and \depth as a function of total budget, measured in tokens. 
Figure \ref{fig:err_d_L_N_plot} shows this tradeoff for all dense transformers in our dataset with parameter counts of at least 2B parameters; results are averaged across models with one run per model. 
For these curves, we use the estimate $\dhat$ of \hidden as an input to the estimate of \depth and \params, which causes error to propagate to $\Lhat$ and $\Phat$.
On average, \textbf{\algo estimates \hidden, \depth, and \params on average to within $\sim 40\%$, $25\%$, and $65\%$ relative error, respectively}, using a total budget of $\sim $ 800B tokens for model sizes ranging from 4B dense to 30B \moe.
While there is significant variability across models (shown as lighter curves in \Cref{fig:err_d_L_N_plot}), we find that most of the models generally improve in relative error as budget increases, as expected.

\begin{wrapfigure}[22]{l}{0.55\textwidth}
	\centering
    \vspace{-1em}
	\includegraphics[width=0.9\linewidth]{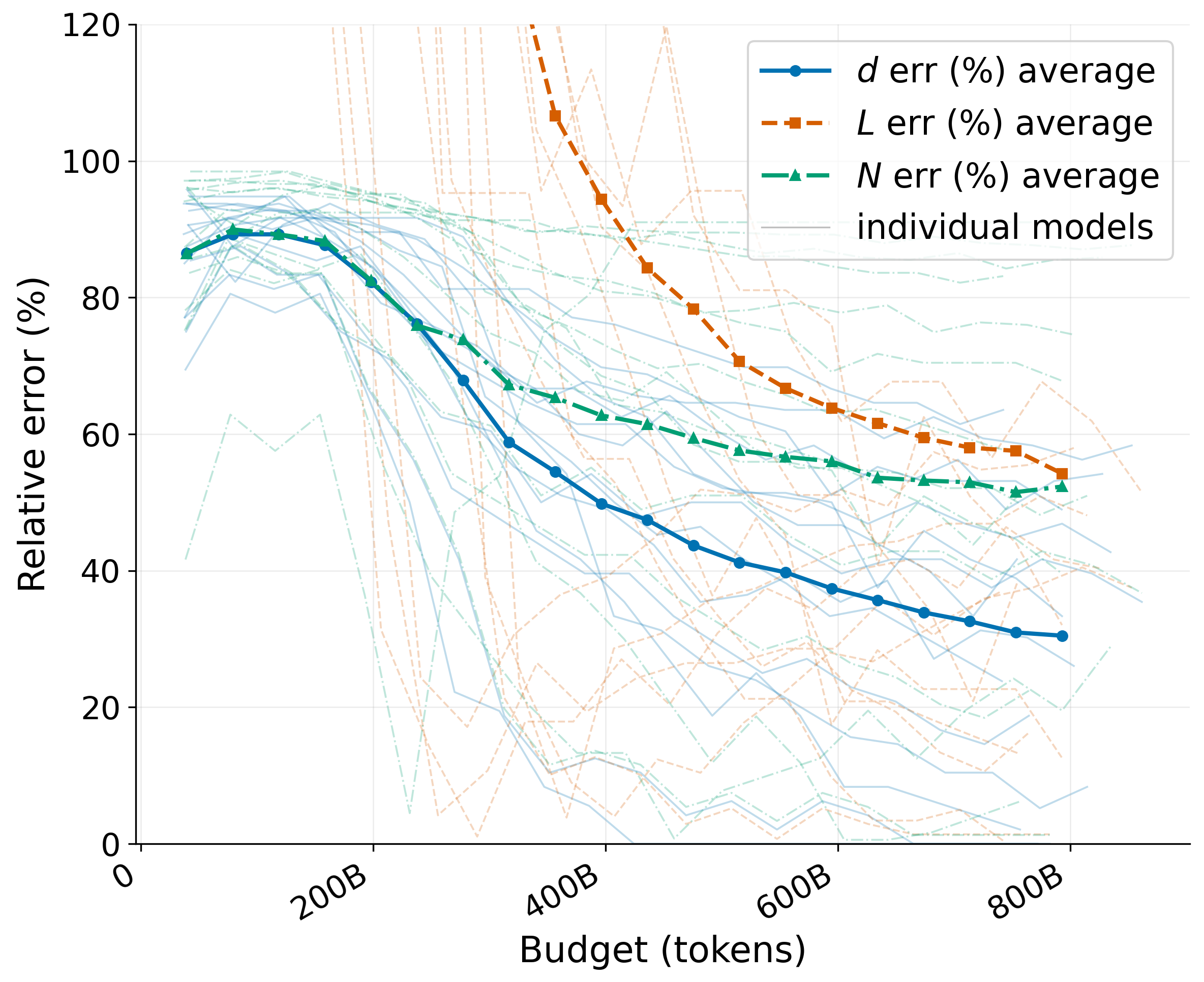}
	\caption{Relative error vs. token budget. Errors for $d$, $L$, and $\param$ are shown as a function of total token budget, with bold curves denoting averages across target models and faint curves denoting individual models. Larger budgets generally reduce error, with $L$ improving fastest and $\param$ remaining harder to estimate. 
    Results include $\geq$3B-parameter dense models and \moe models.
	}
	\label{fig:err_d_L_N_plot}
\end{wrapfigure}

\paragraph{\algo works better on large models.}
Figure \ref{fig:val_plot_n_layers} shows how relative error for estimating \depth, \params, and \hidden, respectively depends on model size for a collection of model architectures.
We see that while the common-set prompting attack works well at all model sizes, 
\depth and \params are more sensitive to model size. In particular, on larger models ($\geq 3$B parameters), \algo achieves low relative error (often below 30\%) for \depth and \params estimation. However, for models below 3B parameters, the estimation error is often over 100\% for both parameters. 
	This is likely because in smaller models, the dominant computational cost is no longer prefill, because the asymptotics have not yet kicked in. 
    For larger models, which are more common in production, \ttft is dominated by prefill costs, so our model is more accurate.

\subsection{Hidden-Dimension Recovery via the Common-Set Prompting Attack}
\label{sec:results_spectral}
\label{sec:eval-common-set}

We evaluate the spectral hidden-dimension estimator on a held-out validation set of 32 models.
Figure~\ref{fig:val_plot} summarizes the per-model rounded relative error, averaged across runs, as a function of total prompt-token budget; the right-most point of each curve corresponds to the maximum sample budget of $N{=}10^{6}$ per prompt (${\sim}800$B prompt tokens).
Each line averages models within a group; each $x$-axis tick is the mean tokens consumed over a batch of 50k samples. We separate models by size (\texttt{small} $\leq{}2$B 
parameters, \texttt{mid} ${\approx} 2$--$4$B, and \texttt{large} $\geq{}7$B) and architecture (dense vs. \moe).

\begin{wrapfigure}[15]{l}{0.6\textwidth}
	\vspace{-1em}
    \centering
	\includegraphics[width=0.9\linewidth]{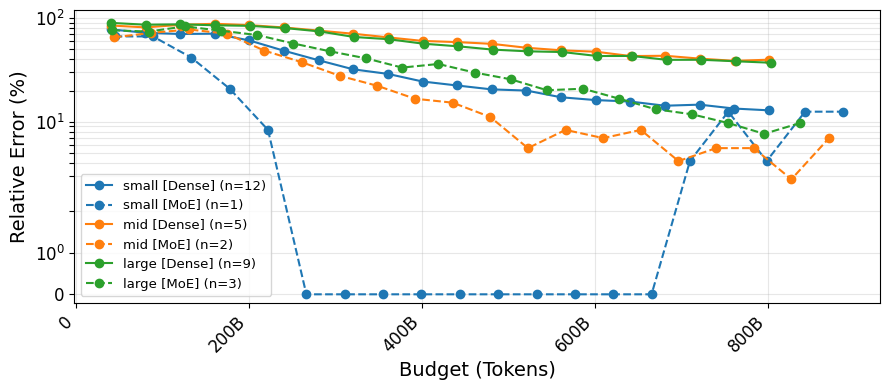}
	\caption{Mean relative error of the spectral hidden-dimension estimate $\dhat$ on the validation set vs.\ total prompt-token budget. Each curve averages over the $n$ validation models in the corresponding size and architecture group. 
    }
	\label{fig:val_plot}
\end{wrapfigure}
\Cref{fig:val_plot} highlights three points: (1) 
Average relative error drops monotonically with budget for all model classes. 
(2) Smaller models converge faster on average. This is expected, as small models tend to have smaller values of $d$, which in turn requires a smaller common set for recovery, and a lower token budget. 
(3) \moe models tend to converge noticeably faster than dense ones.
We hypothesize that this is because the expert routing structure in \moe models sends each token through a small subset of experts, so the effective next-token distribution depends on fewer parameters and this helps saturate the eigenspace quicker than dense models. 

Overall, the most challenging instances for the methodology incur errors between 49-64\% on dense architectures from the Llama, Granite, and Qwen families.
One possible explanation is that our token budget for these models was not high enough to saturate the eigenspectrum, in which case the estimator can only lower-bound the true rank and additional sampling is required to detect the true elbow. \Cref{fig:llama_8b_shape} illustrates this for Llama-3.1-8B ($d^{*}{=}4{,}096$): at our maximum budget, the saddle and peak landmarks that bracket the change-point detector's area of interest both fall well below $d^{*}$, so the spectrum has not saturated enough for the detector to reliably latch onto the true rank.

\paragraph{\algo incurs ${\approx}12-190\times$ the cost of \citet{carlini2024stealing}}
\label{sec:compare-baseline}
The baseline of \citet{carlini2024stealing} assumes access to top-$k$ logits and a logit bias function. Figure~\ref{fig:carlini_vs_us} shows what that extra access buys. For each of the 23 configurations that reach relative error $\le 30\%$, we used $D=20,000$ prompts and grew the per-prompt sample budget until \algo's $\dhat$ reached that threshold, and we plot the cost ratio of our attack at that point divided by the cost of the top-$1$ logit-bias attack in \citet{carlini2024stealing}.
Their cost is $d\,|\vocab|(\bar{\ell}_{\text{sys}}{+}2)$ tokens, counting one input plus one output token per logit-bias query on top of the average system prompt $\bar{\ell}_{\text{sys}}$; \algo's cost is the realized prompt-token total.

\begin{wrapfigure}{r}{0.5\textwidth}
	\vspace{-0.5em}
    \centering
	\includegraphics[width=0.9\linewidth]{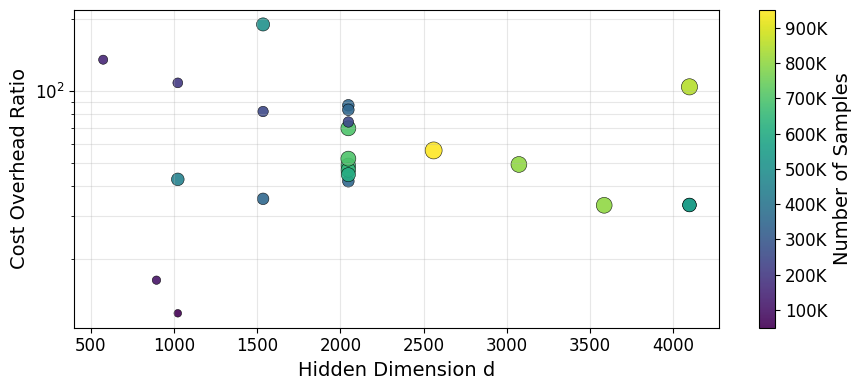}
	\caption{{Cost overhead ratio as a function of \hidden}. Cost overhead ratio = (Our cost)/(Cost of \citet{carlini2024stealing}). We report the cost to reach relative error $\le 30\%$. Marker color encodes the number of samples used to reach that error at each point.}
	\label{fig:carlini_vs_us}
\end{wrapfigure}
Overall, the cost ratio varies with \hidden $d$ but \algo spends about 1 to 2.3 orders of magnitude (median ${\sim}1.7$) more tokens than the logit-bias attack to reach the same accuracy on \hidden.
The overhead is substantial; these numbers partially reflect the inherent cost of working under the weaker single-decoded-logprob threat model, which still applies on production APIs that have removed top-$k$ and logit-bias access.
However, they may also reflect inefficiencies in \algo. For example, there may be more efficient prompting methods, or more efficient attacks altogether, that we have yet to identify.

The overhead shrinks to ${\sim}33\times$ on OLMo-2-7B and OLMo-3-7B at $d{=}4{,}096$. This matches the analysis: our per-prompt sample requirement scales sub-linearly in $d$, while Carlini's query count scales as $d\,|\vocab|$. 

\paragraph{Theory predicts the empirical sample complexity.}
\begin{figure}[htpb]
  \centering
  \includegraphics[width=0.9\linewidth]{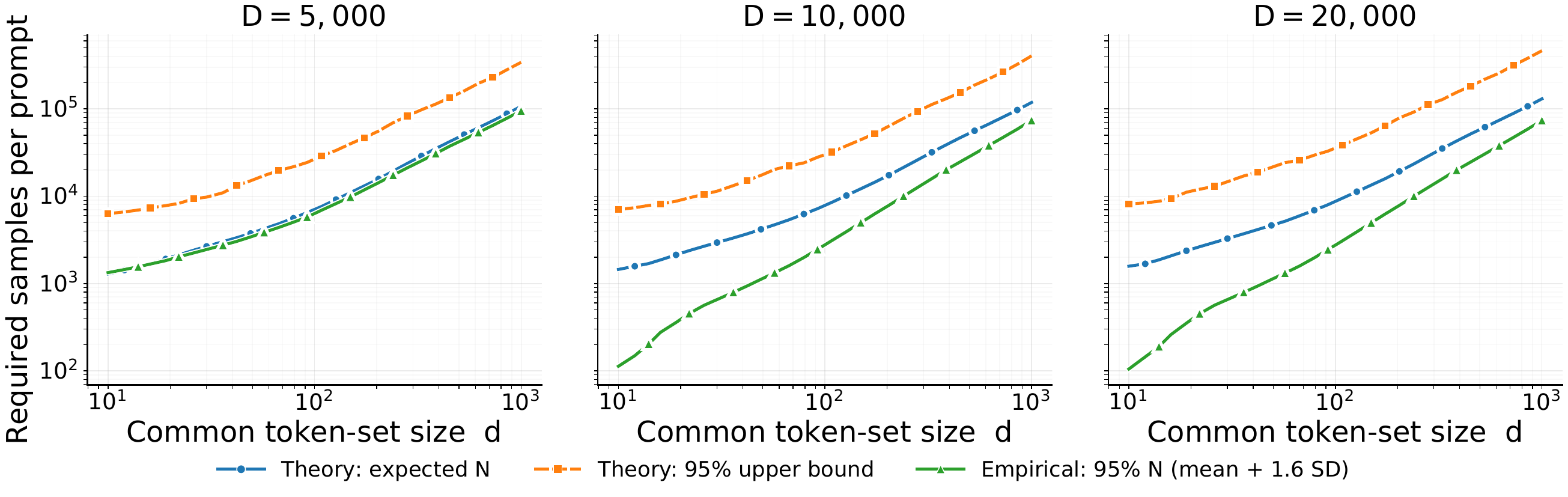}
    \caption{{Required samples per prompt $N$ to reach a common token set of
  size $d$, for $D\in\{5000,10000,20000\}$ prompts.} Token distributions are from
  the smallest HuggingFace SmolLM model ($135$M parameters; vocabulary
  $V=49{,}152$, temperature $T=1$). The two theory lines,
  {expected $N$} (blue) and the
  {$95\%$ upper bound} (orange,
  ~\Cref{cor:t-flat-sample-complexity}), analyze the strict requirement that a
  token appear in \emph{all} $D$ prompts. The
 {empirical $95\%$ $N$} (green) is Common Set Prompting,
  which greedily prunes prompts.}
  \label{fig:expected-N-vs-D}
\end{figure}
Figure~\ref{fig:expected-N-vs-D} reports the samples per prompt $N$ needed to obtain a
common token set of size $d$, as a function of $d$, for $D\in\{5000,10000,20000\}$
prompts on output distributions from the smallest SmolLM model. We draw $N$ tokens from
each prompt and call a token present in a prompt if it is drawn at least once. The
quantity plotted is always $N$ for a target of $d$ shared tokens. The curves differ in
which tokens count as shared. The theory analyzes the \emph{strict} common set, i.e., the
tokens present in all $D$ prompts. Common Set Prompting instead produces a
\emph{pruned} common set as it greedily discards the prompts that are hardest to satisfy
and keeps the tokens shared across the remaining subset, so a token need not appear in
every prompt to be counted. For the strict common set we report two theoretical
quantities: the expected $N$, obtained by inverting the expected intersection size
$\mu(N)=\sum_{i}\prod_{q=1}^{D}(1-(1-p^{(i)}_q)^N)$ at $\mu(N)=d$, and the $95\%$ upper
bound from Corollary~\ref{cor:t-flat-sample-complexity}, the smallest $N$ that
guarantees size at least $d$ with probability $1-\delta=0.95$. For the pruned common
set we report the empirical $95\%$ $N$ as the per-run mean plus $1.645$ standard
deviations over $40$ runs of the algorithm. Since pruning requires agreement across only
a subset rather than all $D$ prompts, the empirical cost lies below the theoretical upper bound and within a small constant factor of it. The theory, though it analyzes the
stricter strict-intersection problem, reliably predicts the sample complexity of Common Set Prompting.

\subsection{Depth and Total Parameter Count Estimation via the Timing Attack}
\label{sec:eval-timing}

We first validate the assumptions made in \Cref{sec:timing}, particularly the scaling relation in \eqref{eq:scaling_relation}. To this end, \Cref{fig:scaling_block_b} shows the end-to-end \ttft as a function of our desired scaling, which we model as being linear in $L\ell^2 d$.
Notice first that for most LLMs, \textbf{a linear timing model is reasonable, if not perfect}, with the same slope $\beta$ and offset $\alpha$ across models. However, the Qwen 3.5 family (shown in red, pink, and purple) requires a different linear fit for prediction.
This suggests that timing attacks could benefit from an initial phase to estimate the model family, possibly via model fingerprinting techniques. We leave this question to future work.

\begin{figure}[htpb]
	\centering
	\begin{minipage}[t]{0.48\textwidth}
		\centering
		\includegraphics[width=\linewidth]{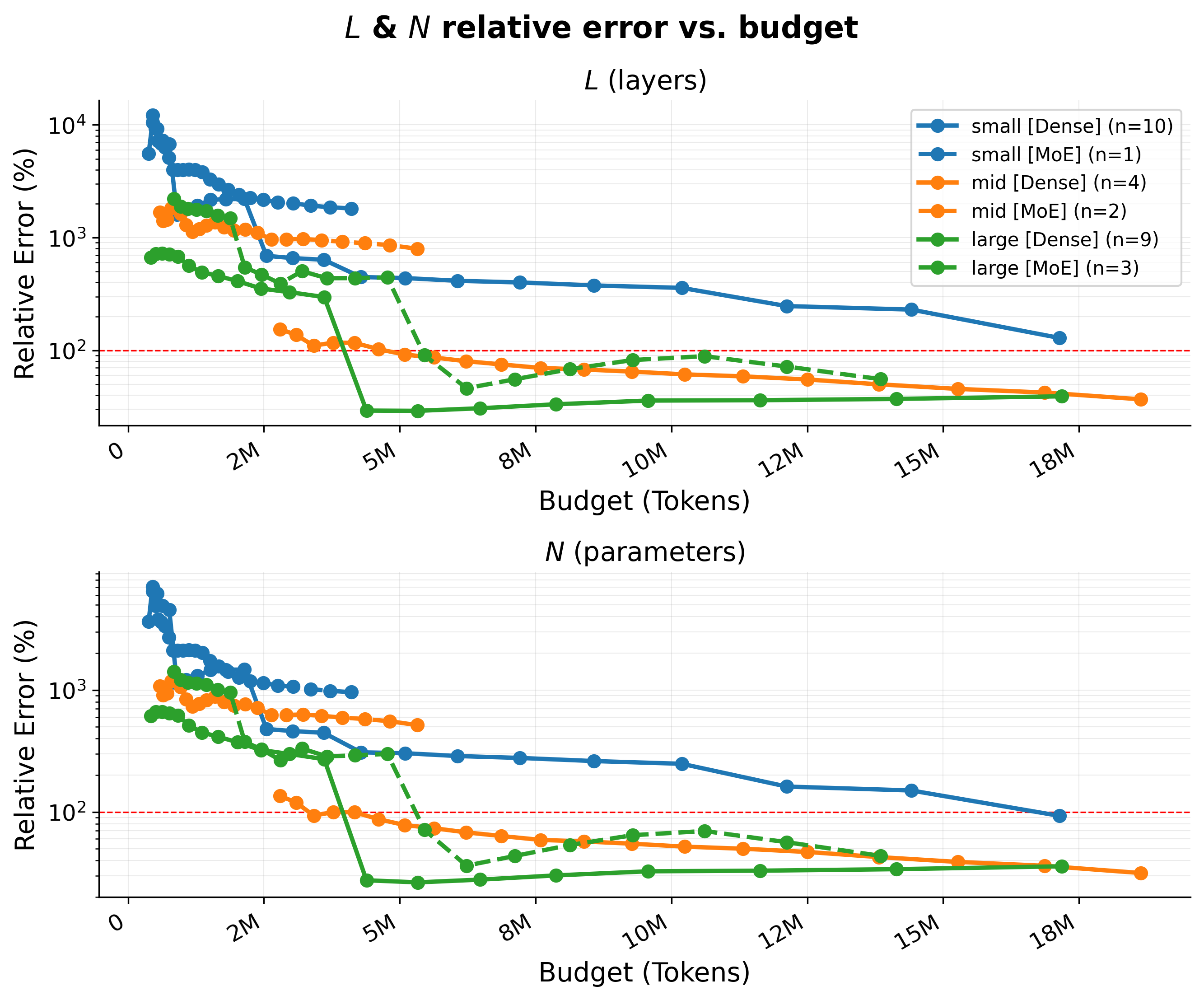}
		\caption{Mean relative error of the number of layers and total model parameter count estimates on the validation set vs.\ total prompt-token budget for timing method. Here we assume perfect knowledge of hidden dimension. Each curve averages over the $n$ validation models in the corresponding size and architecture group.}
		\label{fig:val_plot_timing}
	\end{minipage}
	\hfill
	\begin{minipage}[t]{0.48\textwidth}
		\centering
		\begin{overpic}[width=\linewidth]{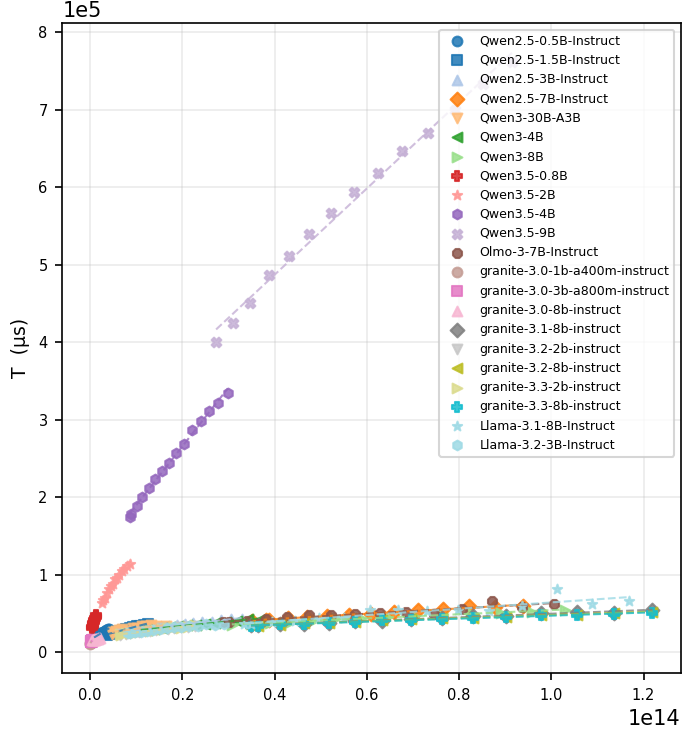}
			\put(50,-0){\makebox[0pt]{{\small $L\ell^2 d$}}}
		\end{overpic}
		\vspace{1em}
		\caption{Measured runtime versus $L\, \ell^{2}\, d$ relation for several models considered in this work. The Qwen3.5 family (points in red, pink, dark purple, and light purple) shows a different slope than others, which fall into LLaMA-style with SwiGLU type of blocks.}
		\label{fig:scaling_block_b}
	\end{minipage}
\end{figure}

Figure~\ref{fig:val_plot_timing} (and Figure~\ref{fig:val_plot_n_layers} in~\Cref{app:additional_timing}) presents the relative error for the layer count and total parameter count. Layer count estimates are obtained by combining the spectral hidden-dimension estimate
with the timing scaling relation of Eq.~\eqref{eq:scaling_relation}. For total parameter estimation, we use Eq.~\eqref{eq:total_param} and further assume that the vocabulary size is known exactly; this is a modeling simplification whose effect is small,
typically contributing less than $2\%$ to the overall estimate. While the average relative error for layer count and total parameters is approximately $50\%$ for models $\geq$ 3B parameter counts, several individual models yield substantially more accurate estimates: $\texttt{Qwen2.5-7B-Instruct}$ (${\sim}16\%$ for $L$, ${\sim}14\%$ for $N$), $\texttt{OLMo-3-7B-Instruct}$ (${\sim}20\%$, ${\sim}18\%$), and $\texttt{Qwen3-8B}$ (${\sim}35\%$, ${\sim}30\%$). The notable exceptions are $\texttt{granite-3.0-3B-A800M}$ and $\texttt{granite-3.0-8B}$, whose limited context windows prevent processing sufficiently long inputs to reach the prefill-dominated regime required by our scaling relation.

\paragraph{Performance does not improve significantly with budget.}
Figure \ref{fig:L_param_err_budget_per_model} shows that for larger models ($\geq$ 3B parameters), the accuracy of our method does not significantly improve with token budget. Recall that our method requires conducting linear regression on timing measurements for input prompts of different lengths; the slope is used to estimate \hidden and \params. 
It appears that we can accurately estimate the slope with relatively few samples at each sequence length. Hence, our estimation is not very sensitive to token budget for large models, at budgets that are orders of magnitude smaller than the requirements for the common-set prompting attack. This explains why we set the token budget to a flat 10M tokens in Figure \ref{fig:overview}.

\section{Conclusion, Limitations, and Future Work}
\label{sec:conclusion}
We present \algo, an attack that recovers certain hyperparameters of transformer-based architectures based only on black-box API access to a target model. Unlike prior attacks, \algo does not make use of top-$k$ logits or a logit-bias function.
We show that even with the very limited information available in modern APIs, we can reconstruct the \hidden, \depth, and \params of many LLMs, up to an accuracy threshold. 

This work has several limitations and directions for future research. First, \algo can incur nontrivial estimation error at high cost: reducing relative error requires token budgets around $10^{11}$ to $10^{12}$ per target (median ${\sim}4\times10^{11}$), which is $12$ to $190$ times more expensive than \citet{carlini2024stealing}, reflecting the cost of using the weaker single-decoded-logprob threat model. Second, our evaluation focuses on standard dense and \moe transformers from a few model families, and we have not yet studied how spectral signals behave in other architectures, such as state-space or hybrid models. Finally, the timing-based estimates of \depth and \params depend on a prefill-dominated regime and a calibrated serving stack. These estimates are less reliable for small models where prefill is not the main bottleneck, and they have not been tested on \moe deployments or in varied serving environments.

\ifpublic
\section*{Acknowledgments}
This effort was sponsored in whole or in part by the United States Government (USG). The U.S. Government is authorized to reproduce and distribute original submissions for publication for Governmental purposes notwithstanding any copyright notation thereon. The views and conclusions contained herein are those of the authors and should not be interpreted as necessarily representing the official policies or endorsements, either expressed or implied, of the United States Government (USG). The authors also thank Srinivasa Pranav for helpful discussions.
\fi
 
\bibliographystyle{plainnat}
\bibliography{main}

@misc{kaplan2020scaling,
      title={Scaling Laws for Neural Language Models}, 
      author={Jared Kaplan and Sam McCandlish and Tom Henighan and Tom B. Brown and Benjamin Chess and Rewon Child and Scott Gray and Alec Radford and Jeffrey Wu and Dario Amodei},
      year={2020},
      eprint={2001.08361},
      archivePrefix={arXiv},
      primaryClass={cs.LG},
      url={https://arxiv.org/abs/2001.08361}, 
}

@inproceedings{yang2024gated,
title={Gated Delta Networks: Improving Mamba2 with Delta Rule},
author={Songlin Yang and Jan Kautz and Ali Hatamizadeh},
booktitle={The Thirteenth International Conference on Learning Representations},
year={2025},
url={https://openreview.net/forum?id=r8H7xhYPwz}
}

@misc{shazeer2020glu,
      title={GLU Variants Improve Transformer}, 
      author={Noam Shazeer},
      year={2020},
      eprint={2002.05202},
      archivePrefix={arXiv},
      primaryClass={cs.LG},
      url={https://arxiv.org/abs/2002.05202}, 
}

@inproceedings{
ainslie2023gqa,
title={{GQA}: Training Generalized Multi-Query Transformer Models from Multi-Head Checkpoints},
author={Joshua Ainslie and James Lee-Thorp and Michiel de Jong and Yury Zemlyanskiy and Federico Lebron and Sumit Sanghai},
booktitle={The 2023 Conference on Empirical Methods in Natural Language Processing},
year={2023},
url={https://openreview.net/forum?id=hmOwOZWzYE}
}

@inproceedings{
dao2022flashattention,
title={FlashAttention: Fast and Memory-Efficient Exact Attention with {IO}-Awareness},
author={Tri Dao and Daniel Y Fu and Stefano Ermon and Atri Rudra and Christopher Re},
booktitle={Advances in Neural Information Processing Systems},
editor={Alice H. Oh and Alekh Agarwal and Danielle Belgrave and Kyunghyun Cho},
year={2022},
url={https://openreview.net/forum?id=H4DqfPSibmx}
}

@article{Duddu2018,
  title={Stealing neural networks via timing side channels},
  author={Duddu, Vasisht and Samanta, Debasis and Rao, D Vijay and Balas, Valentina E},
  journal={arXiv preprint arXiv:1812.11720},
  year={2018}
}

@inproceedings{hu2020deepsniffer,
  title={Deepsniffer: A {DNN} model extraction framework based on learning architectural hints},
  author={Hu, Xing and Liang, Ling and Li, Shuangchen and Deng, Lei and Zuo, Pengfei and Ji, Yu and Xie, Xinfeng and Ding, Yufei and Liu, Chang and Sherwood, Timothy and others},
  booktitle={Proceedings of the Twenty-Fifth International Conference on Architectural Support for Programming Languages and Operating Systems},
  pages={385--399},
  year={2020}
}

@article{zheng2024inputsnatch,
  title={Inputsnatch: Stealing input in {LLM} services via timing side-channel attacks},
  author={Zheng, Xinyao and Han, Husheng and Shi, Shangyi and Fang, Qiyan and Du, Zidong and Hu, Xing and Guo, Qi},
  journal={arXiv preprint arXiv:2411.18191},
  year={2024}
}

@inproceedings{xu2024instructional,
  title={Instructional fingerprinting of large language models},
  author={Xu, Jiashu and Wang, Fei and Ma, Mingyu and Koh, Pang Wei and Xiao, Chaowei and Chen, Muhao},
  booktitle={Proceedings of the 2024 Conference of the North American Chapter of the Association for Computational Linguistics: Human Language Technologies (Volume 1: Long Papers)},
  pages={3277--3306},
  year={2024}
}

@article{JoagProschan83,
 ISSN = {00905364, 21688966},
 URL = {http://www.jstor.org/stable/2240482},
 author = {Kumar Joag-Dev and Frank Proschan},
 journal = {The Annals of Statistics},
 number = {1},
 pages = {286--295},
 publisher = {Institute of Mathematical Statistics},
 title = {Negative Association of Random Variables with Applications},
 urldate = {2026-05-28},
 volume = {11},
 year = {1983}
}

@article{dubhashi1996balls,
  title={Balls and bins: A study in negative dependence},
  author={Dubhashi, Devdatt P and Ranjan, Desh},
  journal={BRICS Report Series},
  volume={3},
  number={25},
  year={1996}
}

@inproceedings{finlayson2024logits,
  title={Logits of {API}-Protected {LLMs} Leak Proprietary Information},
  author={Finlayson, Matthew and Ren, Xiang and Swayamdipta, Swabha},
  booktitle={Conference on Language Modeling (COLM)},
  year={2024},
  url={https://openreview.net/forum?id=oRcYFm8vyB}
}

@article{carlini2024stealing,
  title={Stealing part of a production language model},
  author={Carlini, Nicholas and Paleka, Daniel and Dvijotham, Krishnamurthy Dj and Steinke, Thomas and Hayase, Jonathan and Cooper, A Feder and Lee, Katherine and Jagielski, Matthew and Nasr, Milad and Conmy, Arthur and others},
  journal={arXiv preprint arXiv:2403.06634},
  year={2024}
}

@article{ubaru2016fast,
  title={Fast estimation of approximate matrix ranks using spectral densities},
  author={Ubaru, Shashanka and Saad, Yousef and Seghouane, Abd-Krim},
  journal={arXiv preprint arXiv:1608.05754},
  year={2016}
}

@article{zhu2006automatic,
  title={Automatic dimensionality selection from the scree plot via the use of profile likelihood},
  author={Zhu, Mu and Ghodsi, Ali},
  journal={Computational Statistics \& Data Analysis},
  volume={51},
  number={2},
  pages={918--930},
  year={2006},
  publisher={Elsevier}
}

@article{luxburg2007spectral,
  title={A tutorial on spectral clustering},
  author={von Luxburg, Ulrike},
  journal={Statistics and Computing},
  volume={17},
  number={4},
  pages={395--416},
  year={2007},
  publisher={Springer}
}

@article{shao2025reading,
  title={Reading Between the Lines: Towards Reliable Black-box {LLM} Fingerprinting via Zeroth-order Gradient Estimation},
  author={Shao, Shuo and Li, Yiming and Yao, Hongwei and Chen, Yifei and Yang, Yuchen and Qin, Zhan},
  journal={arXiv preprint arXiv:2510.06605},
  year={2025}
}

@article{shao2025sok,
  title={Sok: Large language model copyright auditing via fingerprinting},
  author={Shao, Shuo and Li, Yiming and He, Yu and Yao, Hongwei and Yang, Wenyuan and Tao, Dacheng and Qin, Zhan},
  journal={arXiv preprint arXiv:2508.19843},
  year={2025}
}

@inproceedings{pasquini2025llmmap,
  title={$\{$LLMmap$\}$: Fingerprinting for large language models},
  author={Pasquini, Dario and Kornaropoulos, Evgenios M and Ateniese, Giuseppe},
  booktitle={34th USENIX Security Symposium (USENIX Security 25)},
  pages={299--318},
  year={2025}
}

@article{zhang2024reef,
  title={Reef: Representation encoding fingerprints for large language models},
  author={Zhang, Jie and Liu, Dongrui and Qian, Chen and Zhang, Linfeng and Liu, Yong and Qiao, Yu and Shao, Jing},
  journal={arXiv preprint arXiv:2410.14273},
  year={2024}
}

@article{zeng2024huref,
  title={Huref: Human-readable fingerprint for large language models},
  author={Zeng, Boyi and Wang, Lizheng and Hu, Yuncong and Xu, Yi and Zhou, Chenghu and Wang, Xinbing and Yu, Yu and Lin, Zhouhan},
  journal={Advances in Neural Information Processing Systems},
  volume={37},
  pages={126332--126362},
  year={2024}
}

@inproceedings{gubri2024trap,
  title={Trap: Targeted random adversarial prompt honeypot for black-box identification},
  author={Gubri, Martin and Ulmer, Dennis and Lee, Hwaran and Yun, Sangdoo and Oh, Seong Joon},
  booktitle={Findings of the Association for Computational Linguistics: ACL 2024},
  pages={11496--11517},
  year={2024}
}

@article{verma2024operationalizing,
  title={Operationalizing a threat model for red-teaming large language models ({LLMs})},
  author={Verma, Apurv and Krishna, Satyapriya and Gehrmann, Sebastian and Seshadri, Madhavan and Pradhan, Anu and Ault, Tom and Barrett, Leslie and Rabinowitz, David and Doucette, John and Phan, NhatHai},
  journal={arXiv preprint arXiv:2407.14937},
  year={2024}
}

@article{pan2025survey,
  title={A survey of {LLM} inference systems},
  author={Pan, James and Li, Guoliang},
  journal={arXiv preprint arXiv:2506.21901},
  year={2025}
}

@article{li2024large,
  title={Large language model inference acceleration: A comprehensive hardware perspective},
  author={Li, Jinhao and Xu, Jiaming and Huang, Shan and Chen, Yonghua and Li, Wen and Liu, Jun and Lian, Yaoxiu and Pan, Jiayi and Ding, Li and Zhou, Hao and others},
  journal={arXiv preprint arXiv:2410.04466},
  year={2024}
}

@article{li2026ikp,
  title={Incompressible Knowledge Probes: Estimating Black-Box {LLM} Parameter Counts via Factual Capacity},
  author={Li, Bojie},
  journal={arXiv preprint arXiv:2604.24827},
  year={2026}
}

@article{alhazbi2025rhythm,
  title={{LLMs} have rhythm: Fingerprinting large language models using inter-token times and network traffic analysis},
  author={Alhazbi, Saeif and Hussain, Ahmed and Oligeri, Gabriele and Papadimitratos, Panos},
  journal={IEEE Open Journal of the Communications Society},
  year={2025},
  publisher={IEEE}
}

@inproceedings{hartenstein2025bridging,
  title={Bridging the Security Gap: An Empirical Analysis of LLM-API Integration Vulnerabilities and Mitigation Strategies},
  author={Hartenstein, Sandro},
  booktitle={Proceedings of the 2025 14th International Conference on Software and Computer Applications},
  pages={90--95},
  year={2025}
}

@inproceedings{liu2026exploring,
  title={Exploring and Exploiting Security Vulnerabilities in Self-Hosted LLM Services},
  author={Liu, Zhihuang and Hu, Ling and Tang, Yonghao and Zhou, Tongqing and Liu, Fang and Cai, Zhiping},
  booktitle={Proceedings of the ACM Web Conference 2026},
  pages={2535--2546},
  year={2026}
}

@inproceedings{amirizaniani2024auditllm,
  title={AuditLLM: A tool for auditing large language models using multiprobe approach},
  author={Amirizaniani, Maryam and Martin, Elias and Roosta, Tanya and Chadha, Aman and Shah, Chirag},
  booktitle={Proceedings of the 33rd ACM International Conference on Information and Knowledge Management},
  pages={5174--5179},
  year={2024}
}

@article{mokander2024auditing,
  title={Auditing large language models: a three-layered approach},
  author={M{\"o}kander, Jakob and Schuett, Jonas and Kirk, Hannah Rose and Floridi, Luciano},
  journal={AI and Ethics},
  volume={4},
  number={4},
  pages={1085--1115},
  year={2024},
  publisher={Springer}
}

@inproceedings{rastogi2023supporting,
  title={Supporting human-ai collaboration in auditing llms with llms},
  author={Rastogi, Charvi and Tulio Ribeiro, Marco and King, Nicholas and Nori, Harsha and Amershi, Saleema},
  booktitle={Proceedings of the 2023 AAAI/ACM Conference on AI, Ethics, and Society},
  pages={913--926},
  year={2023}
}

@inproceedings{tramer2016stealing,
  title={Stealing machine learning models via prediction $\{$APIs$\}$},
  author={Tram{\`e}r, Florian and Zhang, Fan and Juels, Ari and Reiter, Michael K and Ristenpart, Thomas},
  booktitle={25th USENIX security symposium (USENIX Security 16)},
  pages={601--618},
  year={2016}
}

@article{morris2023language,
  title={Language model inversion},
  author={Morris, John X and Zhao, Wenting and Chiu, Justin T and Shmatikov, Vitaly and Rush, Alexander M},
  journal={arXiv preprint arXiv:2311.13647},
  year={2023}
}

@inproceedings{zhao2025survey,
  title={A survey on model extraction attacks and defenses for large language models},
  author={Zhao, Kaixiang and Li, Lincan and Ding, Kaize and Gong, Neil Zhenqiang and Zhao, Yue and Dong, Yushun},
  booktitle={Proceedings of the 31st ACM SIGKDD Conference on Knowledge Discovery and Data Mining V. 2},
  pages={6227--6236},
  year={2025}
}

@article{chauvin2025log,
  title={Log Probability Tracking of LLM APIs},
  author={Chauvin, Timoth{\'e}e and Merrer, Erwan Le and Ta{\"\i}ani, Fran{\c{c}}ois and Tredan, Gilles},
  journal={arXiv preprint arXiv:2512.03816},
  year={2025}
}

@article{vaswani2017attention,
  title={Attention is all you need},
  author={Vaswani, Ashish and Shazeer, Noam and Parmar, Niki and Uszkoreit, Jakob and Jones, Llion and Gomez, Aidan N and Kaiser, {\L}ukasz and Polosukhin, Illia},
  journal={Advances in Neural Information Processing Systems},
  volume={30},
  year={2017}
}

%%%%%%%%%%%%%%%%%%%%%%%%%%%%%%%%%%%%%%%%%%%%%%%%%%%%%%%%%%%%
\newpage 
\appendix
\startcontents[appendix]
\printcontents[appendix]{}{1}{\section*{Appendix}}

\newpage 

\section{Additional Related Work}
\label{app:related}
In addition to the related work from \Cref{sec:related}, a separate and growing body of work on LLM \emph{fingerprinting} \citep{alhazbi2025rhythm,pasquini2025llmmap,zhang2024reef,zeng2024huref,gubri2024trap,shao2025reading, xu2024instructional} aims to identify which model from a known set is being served, rather than to recover architectural parameters of a previously-unseen model; closely related is the application of such fingerprints to copyright auditing \citep{shao2025sok}. These methods rely on per-target calibration data and are therefore not directly comparable to ours. The threat-model landscape for black-box LLM attacks more broadly is taxonomized in \citet{verma2024operationalizing}.

Methodologically, our use of inference timing as an auxiliary signal for architecture inference draws on a longer tradition outside the LLM setting. \citet{Duddu2018} combine end-to-end runtime with architecture search and distillation to recover substitute DNN models, and \citet{hu2020deepsniffer} use learning-based methods on similar signals. Surveys of modern LLM inference systems \citep{pan2025survey,li2024large} provide the hardware and software context that makes timing a useful signal in our setting. While recent timing-based attacks on LLM APIs have demonstrated input recovery from cache effects \citep{zheng2024inputsnatch} rather than architecture extraction, they reflect the same broader trend of treating runtime as exploitable auxiliary information.

\newpage 
\section{Algorithmic Details}

\subsection{Common-Set Attack Details}
\label{app:common-set}

\subsubsection{Finding High-Entropy Prompts}
\label{app:entropy}
Our common-set prompting method seeks prompts whose output distributions are as flat as possible. As discussed in \Cref{sec:prelim}, increasing the temperature
parameter $\tau$ pushes the distribution over tokens toward uniform. Empirically, we find that temperature alone is not enough. When the unscaled distribution is highly peaked, even large $\tau$ leaves substantial probability mass on the top tokens, as illustrated in \Cref{fig:temp_entropy}. For a prompt inducing high entropy over the tokens (top row), the unscaled top-100 tokens hold $68\%$ of the probability mass; temperature scaling spreads this further, dropping the top-100 mass to just $14\%$ and yielding a near-uniform distribution. For a low-entropy prompt (bottom row), the unscaled top-100 tokens hold $95\%$ of the mass, and even after scaling they retain $34\%$, which is still far from uniform. Prompt
selection is hence the dominant lever for vocabulary coverage. We find that out-of-distribution token sequences, i.e., inputs the model has likely not seen during training, reliably produce diffuse predictions, since the model lacks a strong prior over what token should come next. The next subsection describes how we identify such prompts.

\begin{figure}[htpb]
	\centering
	\includegraphics[width=0.8\textwidth]{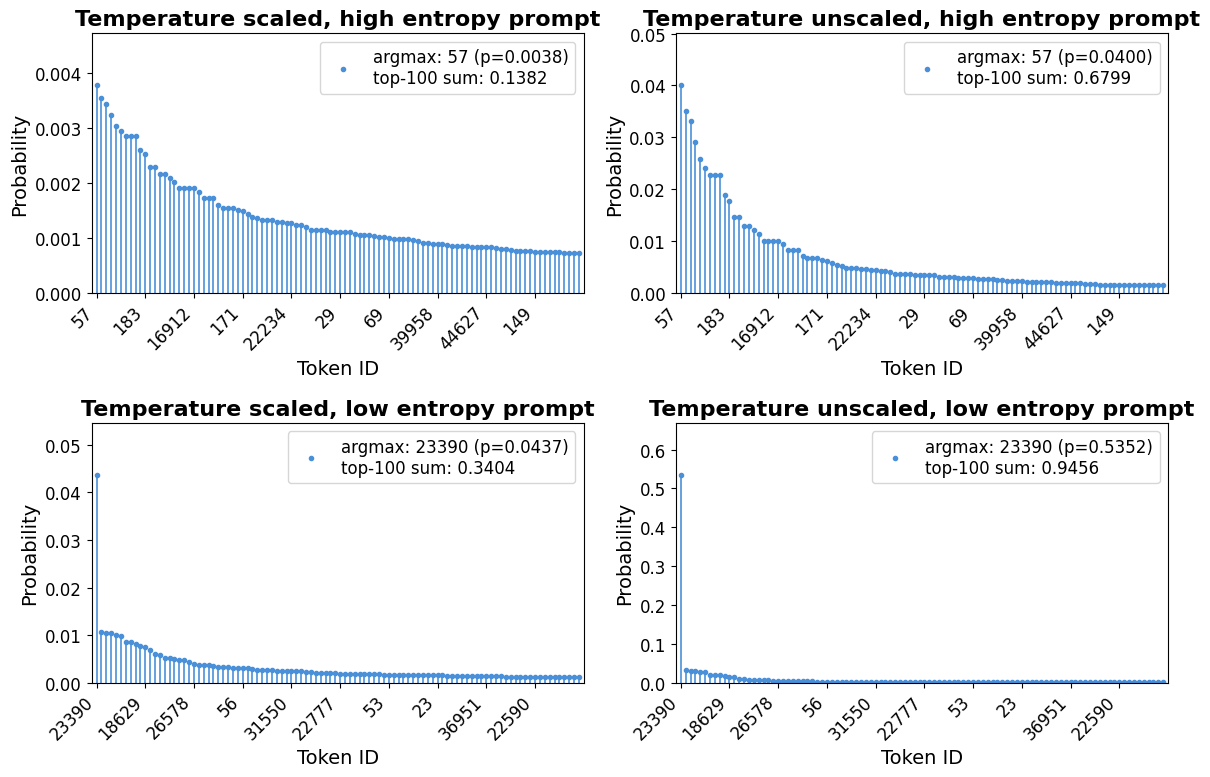}
	\caption{Next-token distributions for a high-entropy prompt (top)
		and a low-entropy prompt (bottom), shown without (left) and with
		(right) temperature scaling. High-entropy prompts yield
		substantially flatter distributions, enabling broader vocabulary
		coverage per sample. }
	\label{fig:temp_entropy}
\end{figure}

\paragraph{Automated High-Entropy Prompt Search}
\label{sec:prompt_search}
High-entropy prompts can be specific to each model, so it is desirable for us to have an automated way of finding enough of them for the common-set method given a new black-box model. In our experiments, our main method of finding these prompts is to sample random prompts from various Unicode character groups. Concretely, we consider character groups like Latin, CJK, emoji, Arabic, Hebrew, Braille, mathematical symbols, etc., and form candidate prompts by randomly sampling characters from these groups. 

Our search runs for a fixed number of iterations and combines exploration with greedy search over a best candidate prompt. The first few rounds draw fresh batches of random prompts across all character groups and record the prompt whose argmax token has the lowest probability seen so far. From the second round onwards, half of each batch consists of single-character variants of the current best prompt, so promising candidates are exploited while new groups can still surface; this exploration phase ends early once any prompt drives the argmax probability below a small threshold. The remaining iterations switch to pure refinement: every candidate is a one-character mutation of the current best, distributed evenly across character groups, and the best prompt is updated greedily after each batch.

All candidates are deduplicated against previously-evaluated prompts to ensure each string is evaluated only once.

After the final iteration, we sort every evaluated prompt by argmax probability and keep the $\inputprompts$ lowest-probability prompts as the prompt set $\promptset$ passed to the common-set construction. In practice, a larger $\inputprompts$ is preferable: more prompts yield broader vocabulary coverage in the common set.

\subsubsection{Greedy NaN Pruning}
\label{app:greedy}
We prune $\hat{\Mlpfull}$ to a large NaN-free submatrix using a greedy heuristic: 
At each step, find the surviving row or column with the highest fraction of NaNs and drop it, then update the per-row and per-column NaN counts. The loop terminates when no NaNs remain. The surviving rows index the $D' \le D$ prompts kept for the spectral method, the surviving columns form the common set $\commonset$, and the submatrix itself is the common-set log-probability matrix $\Mlp \in \mathbb{R}^{D' \times |\commonset|}$. 

\subsubsection{Inflection-Point Detection}
\label{app:changepoint}
Under the standard transformer output head, the logit for token $v_i$ given prompt $\prompt_q$ is $\zqi = \mathbf{h}_q^\top \mathbf{w}_i$ with $\mathbf{h}_q, \mathbf{w}_i \in \mathbb{R}^{d}$, so the logit matrix $\mathbf{Z}$ has rank at most $d$. The log-probability matrix $\Mlp$ (with entries $\yqi$) differs from $\mathbf{Z}$ only by a row-wise log-partition term $\log\sum_j \exp(z_q^{(j)})$, which is identical across all columns within a given row; temperature scaling preserves this structure. The rank-$d$ signal in $\mathbf{Z}$ is therefore recoverable from $\Mlp$ up to a row-wise additive constant.

Deployed models run inference in reduced-precision formats such as \texttt{bfloat16} (8-bit exponent, 7-bit mantissa, $\approx 3$ decimal digits of precision). The logits $\zqi$ are therefore quantized, and the returned log-probabilities carry rounding noise. As \citet{carlini2024stealing} observe, this means $\Mlp$ is essentially always full rank and we must recover its \emph{practical numerical rank}. We model this as $\Mlp = \Mlpstar + \boldsymbol{\eta}$, with a low-rank \emph{head} $\Mlpstar$ of rank at most $d{+}1$ holding the signal from $\mathbf{Z}$, and a full-rank floating-point \emph{noise tail} $\boldsymbol{\eta}$, and recovering $d$ amounts to separating the two. To do this, we take the Gram matrix $\Mgram = \Mlp\Mlp^\top$ when $D' \le |\commonset|$, and $\Mgram = \Mlp^\top\Mlp$ otherwise. Its eigendecomposition $\Mgram = \mathbf{U}\,\boldsymbol{\Lambda}\,\mathbf{U}^\top$ produces ordered eigenvalues $\lambda_1 \geq \lambda_2 \geq \cdots \geq \lambda_{\min(D',|\commonset|)} \geq 0$.

Formally, this is a numerical rank estimation problem \citep{ubaru2016fast}: we fix a threshold $\varepsilon \geq 0$ and report the $\varepsilon$-rank $|\{i : \eigi > \varepsilon\}|$, the index that divides the signal eigenvalues of $\Mlpstar$ from the noise tail of $\boldsymbol{\eta}$. The standard recipes for choosing $\varepsilon$ assume structure our spectrum lacks: a dominant eigengap or a low-density valley in the spectral density separating signal from noise \citep{luxburg2007spectral, ubaru2016fast}, or signal and noise eigenvalues that each concentrate around a level \citep{zhu2006automatic}.

Our goal is to identify the inflection point in this sequence, as illustrated in \Cref{fig:spectrum_compare}. The contrast with prior work motivates the design of our detector. The logit-bias attack of \citet{carlini2024stealing} reconstructs nearly all entries of the dense logit matrix, so the spectrum exhibits a sharp drop of several orders of magnitude precisely at index $d$ (\Cref{fig:spectrum_compare}a), and a simple elbow rule suffices. In our restricted-API setting we observe only sampled log-probabilities over the common set $\commonset$, and the resulting spectrum decays smoothly: the head merges into the noise tail through a gradual transition rather than a visible drop-off (\Cref{fig:spectrum_compare}b), so $d$ must be recovered from the curvature of the transition rather than from a discontinuity in magnitude.

\begin{figure}[htpb]
    \centering
    \begin{minipage}[t]{0.48\textwidth}
        \centering
        \includegraphics[width=\linewidth]{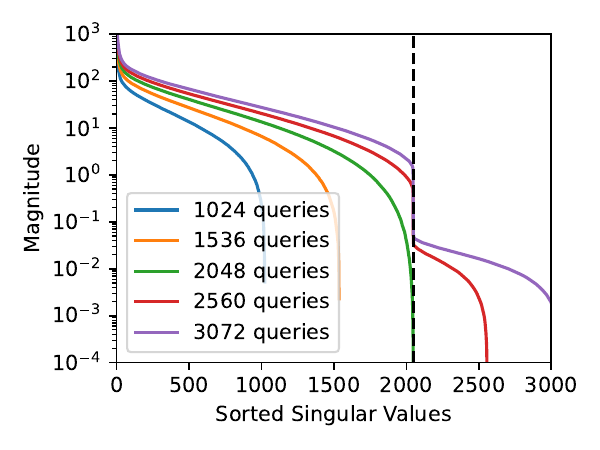}\\[2pt]
        \small (a) Logit-bias attack of \citet{carlini2024stealing}.
    \end{minipage}\hfill
    \begin{minipage}[t]{0.48\textwidth}
        \centering
        \includegraphics[width=\linewidth]{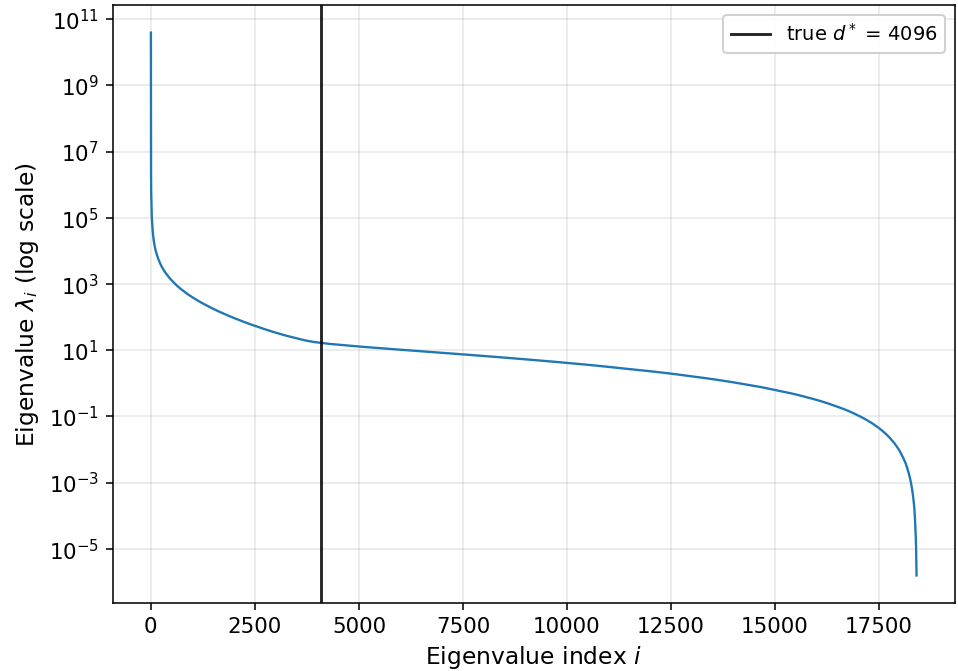}\\[2pt]
        \small (b) Our common-set spectrum (OLMo-3 7B).
    \end{minipage}\\[6pt]
    \begin{minipage}{\textwidth}
        \centering
        \includegraphics[width=\linewidth]{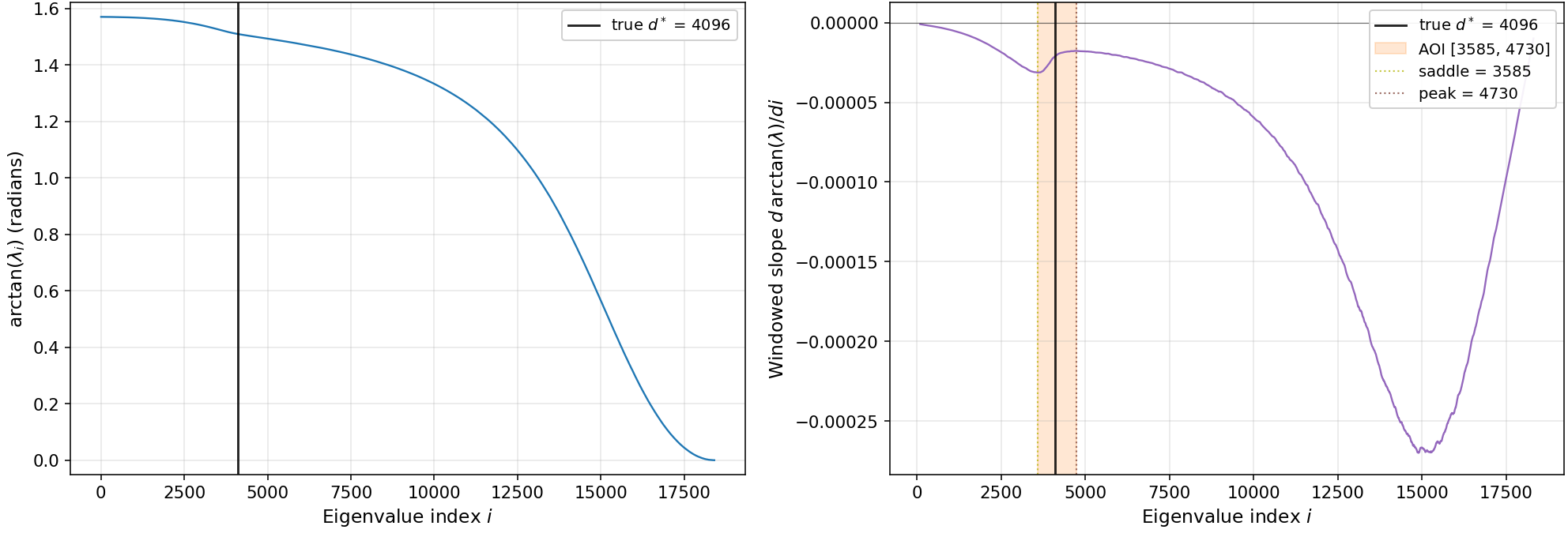}\\[2pt]
        \makebox[0.5\linewidth][c]{\small (c) Arctan-transformed spectrum.}%
        \makebox[0.5\linewidth][c]{\small (d) Windowed slope $\partial_i a_i$ with detected saddle, peak, and \aoi.}
    \end{minipage}
    \caption{Spectra used to recover the hidden dimension $d$. (a) Under the logit-bias attack of \citet{carlini2024stealing}, the singular-value spectrum drops sharply at index $d$ (dashed line) once the number of queries exceeds $d$, and a simple elbow rule recovers $d$. (b) In our restricted-API common-set setting (OLMo-3 7B, true $\dtrue = 4096$), the eigenvalue spectrum decays smoothly across several orders of magnitude with no clean elbow. (c) Applying the element-wise transform $a_i = \arctan(\eigi)$ compresses the head toward the plateau at $\pi/2$ and exposes the head-to-tail transition to slope-based detection. (d) The windowed slope $\partial_i a_i$ exhibits a saddle (local minimum) and a subsequent peak (first non-increase) that bracket the area of interest (\aoi) containing $d$; within this \aoi the algorithm returns the index of maximum negative second difference as $\dhat$, which closely tracks the true $\dtrue = 4096$.}
    \label{fig:spectrum_compare}
\end{figure}
We first apply the element-wise transform
\begin{align}
	a_i = \arctan(\eigi), \qquad i = 1, \dots, \min(D', |\commonset|),
\end{align}
which compresses the signal eigenvalues toward a plateau bounded above by $\frac{\pi}{2}$ and leaves the noise tail as a smooth decay to a final linear drop off at extremely small $a_i$ (\Cref{fig:spectrum_compare}c). This transform exposes the transition region between the head and its decay to slope-based detection when its head is sufficiently saturated.

Let $n = \min(D',|\commonset|)$. We compute the windowed least-squares slope $\partial_i a_i$ over $[i-w, i+w]$ with $w = \max(20, \lfloor n/50 \rfloor)$, then smooth it with a centered moving average $\slope = \mathrm{MA}_{w_s}(\partial_i a_i)$ of half-window $w_s = \max(20, \lfloor n/100 \rfloor)$. As shown in \Cref{fig:spectrum_compare}d, $\tilde{s}$ exhibits a shallow local minimum near the head-to-tail transition followed by a recovery toward zero before the deep negative excursion of the noise floor. We exploit this structure to localize $d$ between two landmarks of $\tilde{s}$, which together delimit the \emph{area of interest} (\aoi) shaded in \Cref{fig:spectrum_compare}d:
\begin{itemize}\setlength\itemsep{0pt}
	\item \textbf{Saddle:} $\isad$ is the first local minimum of $\tilde{s}$ with prominence at least $p\cdot(\max \tilde{s} - \min \tilde{s})$, where $a_i \le \arctan(\pi)$. This marks the lower edge of the \aoi (\Cref{fig:spectrum_compare}d, dotted line at $\isad=3585$ for OLMo-3 7B).
	\item \textbf{Peak:} $\ipeak$ is the first non-increase of $\tilde{s}$ after $\isad$, with $a_{\ipeak} \le \arctan(\pi/2)$ (near the noise floor). This marks the upper edge of the \aoi (\Cref{fig:spectrum_compare}d, dotted line at $\ipeak=4730$).
\end{itemize}
Within this bracket we recover the geometric inflection of the head-saturated $\arctan$ curve via the maximum negative second difference,
\begin{align}
	\dhat = \arg\max_{k \,\in\, [\isad,\, \ipeak]} \bigl(-\Delta^2 a_k\bigr).
\end{align}
For the OLMo-3 7B example in \Cref{fig:spectrum_compare}, the \aoi $[3585, 4730]$ tightly brackets the true hidden dimension $\dtrue = 4096$, and the second-difference maximum within the \aoi recovers $\dhat$ to within a handful of indices.

\newpage
\subsection{Hidden Dimension Attack}
\label{app:hidden-dim-attack}

\begin{algorithm}[htpb]
	\caption{\algo: common-set construction via greedy NaN-pruning}
	\label{alg:commonset-prune}
	\begin{algorithmic}[1]
		\Require Black-box model $f$ with sampled-token + log-prob API only; high-entropy prompt set $\promptset = \{\prompt_q\}_{q=1}^{D}$; samples per prompt $N$; temperature $\tau$.
		\Ensure Pruned log-probability matrix $\Mlp \in \mathbb{R}^{D' \times |\commonset|}$ over the common set $\commonset$.
		\For{$q = 1, \dots, D$}
		\State Draw $N$ samples from $f(\prompt_q)$ at temperature $\tau$, recording each $(v_{k}, y_q^{(k)})$.
		\State Set $\tokenset_q \gets \{\text{unique tokens observed under } \prompt_q\}$.
		\EndFor
		\State Assemble sparse $\hat{\Mlpfull} \in \mathbb{R}^{D \times |\vocab|}$ with $[\hat{\Mlpfull}]_{q,i} = \yqi$ for $v_i \in \tokenset_q$, $\mathrm{NaN}$ otherwise.
		\State Greedy NaN-prune $\hat{\Mlpfull}$: repeatedly drop the surviving row or column with the highest NaN fraction until no NaNs remain. Surviving columns form the common set $\commonset$; surviving rows index the prompts retained for spectral analysis (\Cref{sec:ctm}).
		\State Set $\Mlp \in \mathbb{R}^{D' \times |\commonset|}$ to the pruned submatrix.
		\State \Return pruned matrix $\Mlp$ and common set $\commonset$.
	\end{algorithmic}
\end{algorithm}

\begin{algorithm}[htpb]
	\caption{\algo: hidden-dimension estimation from the pruned spectrum}
	\label{alg:commonset-dhat}
	\begin{algorithmic}[1]
		\Require Pruned log-probability matrix $\Mlp \in \mathbb{R}^{D' \times |\commonset|}$ from \Cref{alg:commonset-prune}; slope half-window $w$; smoothing window $w_s$; prominence fraction $p$; grid $g$.
		\Ensure Estimated hidden dimension $\dhat$.
		\State Form the smaller-side Gram matrix $\Mgram \gets \Mlp\Mlp^{\top}$ if $D' \le |\commonset|$, else $\Mgram \gets \Mlp^{\top}\Mlp$.
		\State Eigendecompose: $\Mgram = \mathbf{U}\,\boldsymbol{\Lambda}\,\mathbf{U}^{\top}$, take $\lambda_1 \geq \cdots \geq \lambda_{\min(D',|\commonset|)}$.
		\State Transform: $a_i \gets \arctan(\eigi)$.
		%FIND AOI
		\State Compute the smoothed first derivative $\slope \gets \mathrm{MA}_{w_s}\!\big(\partial_i a_i\big)$, where $\partial_i a_i$ is the least-squares slope of $a$ on $[i-w,\, i+w]$ with $w = \max(20, \lfloor n/50 \rfloor)$ and $\mathrm{MA}_{w_s}$ is a centered moving average with $w_s = \max(20, \lfloor n/100 \rfloor)$.
		\State Find saddle: $\isad \gets$ first local min of $\tilde{s}$ with prominence $\ge p\cdot(\max\tilde{s}-\min\tilde{s})$, where $a_i = \arctan(\eigi) \le \arctan(\pi)$.
		\State Find Peak: $\ipeak \gets$ first non-increase of $\tilde{s}$ after $\isad$, with $a_{\ipeak} \le \arctan(\pi/2)$ (near the noise floor).
		\State $\dhat \gets \arg\max_{k\,\in\,[\isad,\,\ipeak]} \big(-\Delta^2 a_k\big)$, the geometric inflection of the head-saturated $\arctan$ curve via second differences.
		\State \Return $\dhat_{\mathrm{rnd}} = g\,\lceil \dhat/g \rceil$ \Comment{ceiling-snap to the grid}
	\end{algorithmic}
\end{algorithm}
\vspace{-1em}
\paragraph{\algo hyperparameters}
\label{app:hyperparams}
The hyperparameters in \Cref{alg:commonset-prune,alg:commonset-dhat} are fixed across all models in our evaluation, with values chosen on the development set (\Cref{app:models}):
\begin{itemize}[leftmargin=*, itemsep = -0.1pt]
    \item \textbf{Number of prompts $D$}: $D = 20{,}000$ high-entropy prompts produced by the search procedure of \Cref{app:entropy}.
    \item \textbf{Samples per prompt $N$}: swept up to $N = 10^{6}$ in the budget curves (\Cref{fig:err_d_L_N_plot}); per-model results in \Cref{tab:model-architectures} use the maximum budget.
    \item \textbf{Sampling temperature $\tau$}: $\tau = 2.0$, applied at the API to flatten the next-token distribution.
    \item \textbf{Prompt length}: $5$ Unicode characters per prompt, sampled from the character groups described in \Cref{app:entropy} (not $5$ tokens; tokenized length varies by model).
    \item \textbf{Slope half-window $w$}: $w = \max(20,\, \lfloor n/50 \rfloor)$, where $n = \min(D',\, |\commonset|)$ is the spectrum length after pruning.
    \item \textbf{Smoothing window $w_s$}: $w_s = \max(20,\, \lfloor n/100 \rfloor)$ for the centered moving average applied to the slope.
    \item \textbf{Prominence fraction $p$}: $p = 0.54$, used to gate saddle detection on $\tilde{s}$.
    \item \textbf{Grid spacing $g$}: $g = 128$, the ceiling-snap applied to $\dhat$ (chosen to match the alignment of hidden dimensions in mainstream transformer configurations).
    \item \textbf{Model precision}: \texttt{bfloat16} for all log-probability collection, matching the precision used for serving.
\end{itemize}

\newpage 
\subsection{Timing Attack Details}
\label{app:timing}

\subsubsection{Overall Algorithm}
\label{app:overall-timing}
The end-to-end algorithm for our timing attack can be found in \Cref{fig:algo_timing}.

\begin{figure}[htpb]
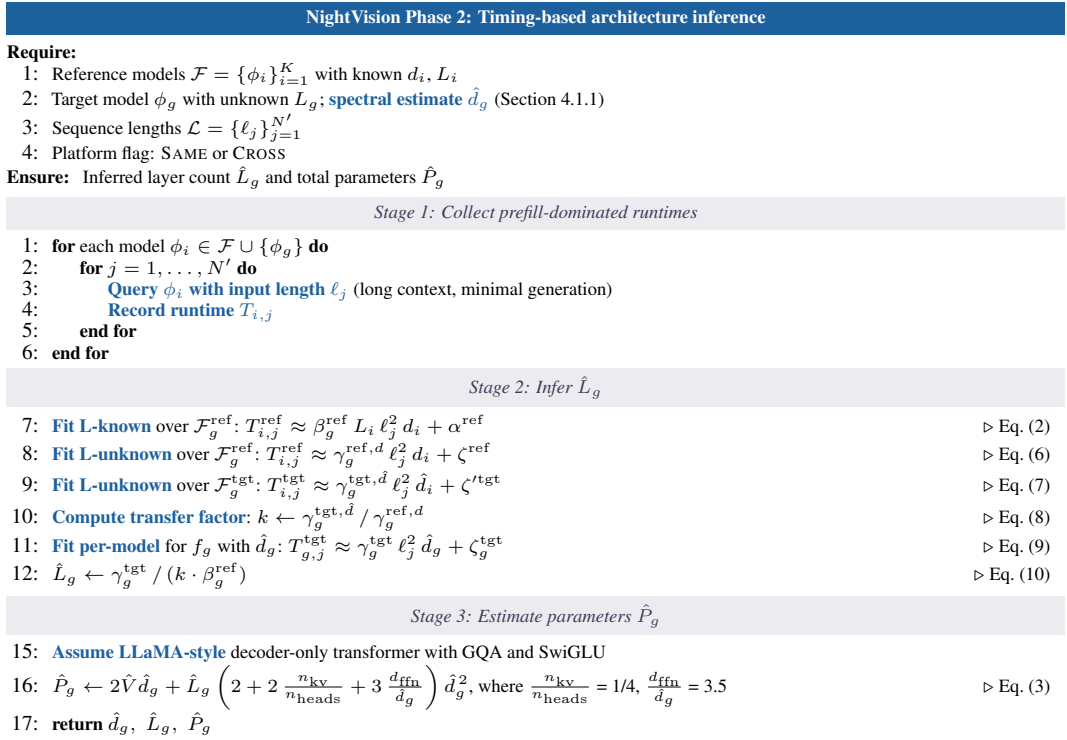

	\begin{tcolorbox}[
		colback=white, colframe=white, boxrule=0.6pt, arc=2pt,
		left=2pt, right=2pt, top=0pt, bottom=2pt,
		boxsep=0pt,
		]
		
		\noindent
		\colorbox{ourshl}{\makebox[\linewidth][c]{%
				\rule[-0.05em]{0pt}{0.1em}\color{white}\bfseries\scriptsize
				\algo Phase 2: Timing-based architecture inference}}
		
		\par\vspace{1pt}
		\noindent
		\begin{algorithmic}[1]
			\scriptsize
			\Require \\
			Reference models $\refset = \{\model_i\}_{i=1}^{K}$ with known $d_i$, $L_i$\\
			Target model $\model_g$ with unknown $L_g$;
			\hlO{spectral estimate $\dhat_g$} (\Cref{sec:spectral})\\
			Sequence lengths $\mathcal{L} = \{\ell_j\}_{j=1}^{N'}$\\
			Platform flag: \textsc{Same} or \textsc{Cross}
			\Ensure Inferred layer count $\Lhat_g$ and total parameters $\Phat_g$
		\end{algorithmic}
		
		\stagebanner{Stage 1: Collect prefill-dominated runtimes}
		
		\noindent
		\begin{algorithmic}[1]
			\scriptsize
			\setcounter{ALG@line}{0}
			\For{each model $\model_i \in \refset \cup \{\model_g\}$}
			\For{$j = 1, \dots, N'$}
			\State \hlO{Query $\model_i$ with input length $\ell_j$}
			(long context, minimal generation)
			\State \hlO{Record runtime $\Tij$}
			\EndFor
			\EndFor
		\end{algorithmic}

		\stagebanner{Stage 2: Infer $\Lhat_g$}
		
		\noindent
		\begin{algorithmic}[1]
			\scriptsize
			\setcounter{ALG@line}{6}
			\State \hlO{Fit L-known} over $\refsetref_{g}$:
			$\Tref \approx \betaref_{g}\, L_i\,
			\ell_j^{2}\, d_i + \alpharef$
			\hfill $\triangleright$ Eq.~\eqref{eq:scaling_relation}
			\State \hlO{Fit L-unknown} over $\refsetref_{g}$:
			$\Tref \approx \gamma^{\mathrm{ref},d}_{g}\,
			\ell_j^{2}\, d_i + \zetaref$
			\hfill $\triangleright$ Eq.~\eqref{eq:reference_scaling}
			\State \hlO{Fit L-unknown} over $\refsettgt_{g}$:
			$\Ttgt \approx \gamma^{\mathrm{tgt},\dhat}_{g}\,
			\ell_j^{2}\, \dhat_i + \zeta^{\prime\mathrm{tgt}}$
			\hfill $\triangleright$ Eq.~\eqref{eq:per_model_tgt}
			\State \hlO{Compute transfer factor}:
			$\transferk \gets \gamma^{\mathrm{tgt},\dhat}_{g} \,/\,
			\gamma^{\mathrm{ref},d}_{g}$
			\hfill $\triangleright$ Eq.~\eqref{eq:transfer_factor}
			\State \hlO{Fit per-model} for $f_g$ with $\dhat_g$:
			$T_{g,j}^{\mathrm{tgt}} \approx \gammatgt_g\,
			\ell_j^{2}\, \dhat_g + \zetatgt_g$
			\hfill $\triangleright$ Eq.~\eqref{eq:cross_platform}
			\State $\Lhat_g \gets \gammatgt_g \,/\,
			(\transferk \cdot \betaref_{g})$
			\hfill $\triangleright$ Eq.~\eqref{eq:L_inference_global}
			
		\end{algorithmic}

		\stagebanner{Stage 3: Estimate parameters $\Phat_g$}
		
		\noindent
		\begin{algorithmic}[1]
			\scriptsize
			\setcounter{ALG@line}{14}
			\State \hlO{Assume LLaMA-style} decoder-only transformer with \gqa and SwiGLU
			\State $\Phat_g \gets 2\hat{V}\dhat_g + \Lhat_g
			\left(2 + 2\,\frac{\nkv}{\nheads}
			+ 3\,\frac{\dffn}{\dhat_g}\right)\dhat_g^{\,2}$, where ${\nkv \over \nheads}$ = 1/4, ${\dffn \over \dhat_g}$ = 3.5
			\hfill $\triangleright$ Eq.~\eqref{eq:total_param}
			\State \Return $\dhat_g,\; \Lhat_g,\; \Phat_g$
		\end{algorithmic}
		
	\end{tcolorbox}
	\caption{Timing-based architecture inference pipeline. The procedure fits
		prefill-dominated scaling relations to infer layer count $\Lhat_g$ and
		total parameters $\Phat_g$ from end-to-end runtime measurements, with
		optional cross-platform calibration.}
	\label{fig:algo_timing}
\end{figure}

\subsubsection{Effect of model architecture on accuracy}
\label{app:architecture-accuracy}
In practice, models largely obey the same scaling relation within each of three architectural categories: \textbf{(A)}~GPT-style with GELU and a two-matrix FFN, \textbf{(B)}~LLaMA-style with SwiGLU and a three-matrix gated FFN, and \textbf{(C)}~Variants, such as Qwen3.5, for which it contains a hybrid combining LLaMA-style
transformer blocks with DeltaNet linear-attention layers~\citep{yang2024gated}. Most modern LLMs
fall into category~(B), which dominates our evaluation. Qwen3.5 requires separate treatment (e.g., see Fig~\ref{fig:scaling_block_b}) because
its DeltaNet layers replace quadratic $O(\ell^2)$ attention with a linear-complexity recurrence.
Consequently, only a subset of its layers exhibit the $\ell^2$-dependent cost assumed by
Eq.~\eqref{eq:scaling_relation}, reducing the effective layer count below the total architectural
depth and necessitating category-specific coefficients.

As shown in Fig~\ref{fig:scaling_law}, category~(A) exhibits noticeably poorer adherence to the scaling relation than category~(B). We attribute this to the structural difference in their FFN blocks: GPT-style models use a two-matrix FFN with GELU, whereas SwiGLU-based models use a three-matrix gated FFN ~\citep{shazeer2020glu}, resulting in 50\% more general matrix--matrix multiplications (GEMMs) per layer. This higher GEMM fraction means compute-bound operations dominate a larger share of prefill wall-clock time in category~(B), making the quadratic scaling ansatz in Eq.~\eqref{eq:scaling_relation} a tighter fit. In category~(A), the smaller GEMM share allows non-scaling overhead to perturb the relation more substantially.

\begin{figure}[t]
 	\centering
 	\includegraphics[width=13cm]{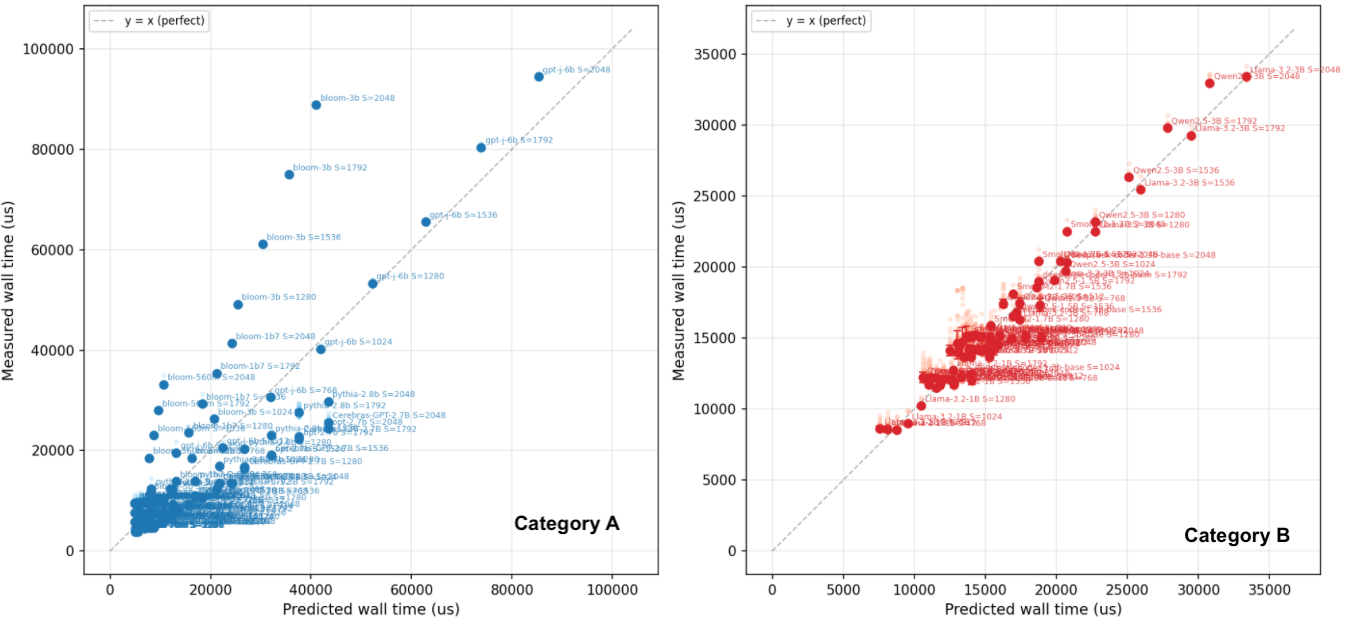}
 	\caption{\footnotesize Measured runtime versus predicted runtime with scaling relation as function of number of layers, hidden dimension, and input sequence length. 
 	}
 	\label{fig:scaling_law}
\end{figure}

\subsubsection{Inferring  Depth $L$ under Hardware Mismatch}
\label{app:cross-hardware}
In practice, the target model is rarely benchmarked on the same platform and configuration used to establish the reference scaling relation. We therefore introduce a cross-platform calibration procedure that transfers the reference relation to a new target platform via a single global scale factor. The same framework also applies to same-platform settings, accounting for residual system-level variations such as inference configuration and runtime conditions.

Concretely, we establish fitted global scaling relations
Eq.~\eqref{eq:scaling_relation} on a
reference platform using a set of reference models
$\mathcal{F}^{\mathrm{ref}} = \{\model_i^{\mathrm{ref}}\}_{i=1}^{K}$ with known $d_r$ and $L_r$:
\begin{equation}
	T_{r,j} \approx \beta^{\mathrm{ref}}\, L_{r}\, \ell_{j}^{2}\, d_{r}
	+ \alpha^{\mathrm{ref}}\approx \gamma^{\mathrm{ref}}\, \ell_{j}^{2}\, d_{r}
	+ \zeta^{\mathrm{ref}},
	\label{eq:reference_scaling}
\end{equation}

\begin{figure}[t]
 	\centering
 	\includegraphics[width=12cm]{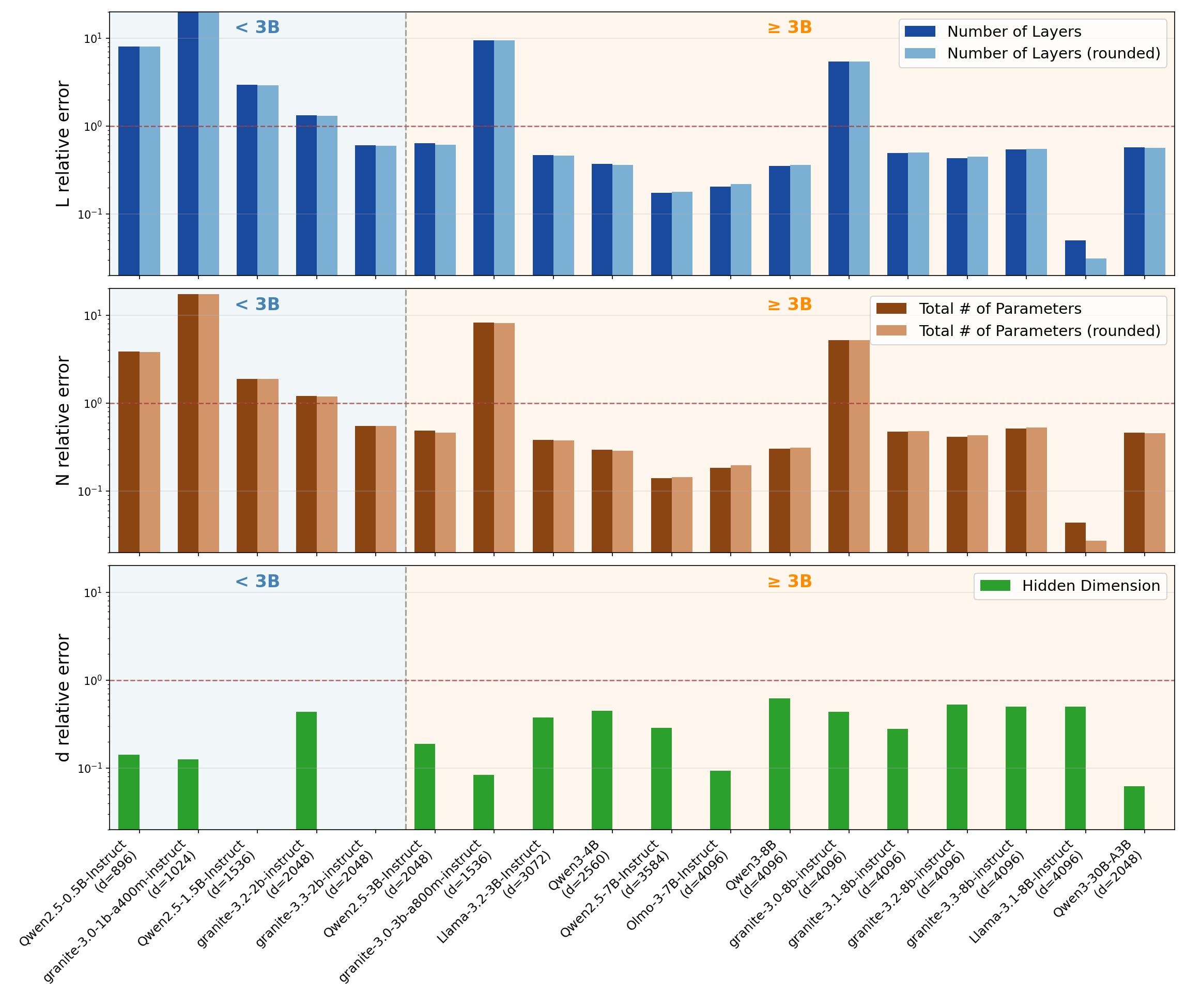}
 	\caption{\footnotesize Relative error for the estimate of the number of layers (top panel), total number of parameters (middle panel), and hidden dimension (bottom panel) for various models. Here we assume the reference and target platform are the same (both with NVIDIA H200 GPU). Dark bars show the raw estimate; light bars show the estimate after rounding to the nearest integer. Some of the largest errors are driven by small models with $<$ 3B parameters. For the layer-count and total-parameter estimates, the large errors for $\texttt{granite-3.0-3b-a800m}$ and $\texttt{granite-3.0-8b}$ arise from their small context window of 4096 tokens, which is insufficient to effectively probe the prefill stage.
 	}
 	\label{fig:val_plot_n_layers}
\end{figure}

\begin{figure}[t]
	\centering
	\includegraphics[width=12cm]{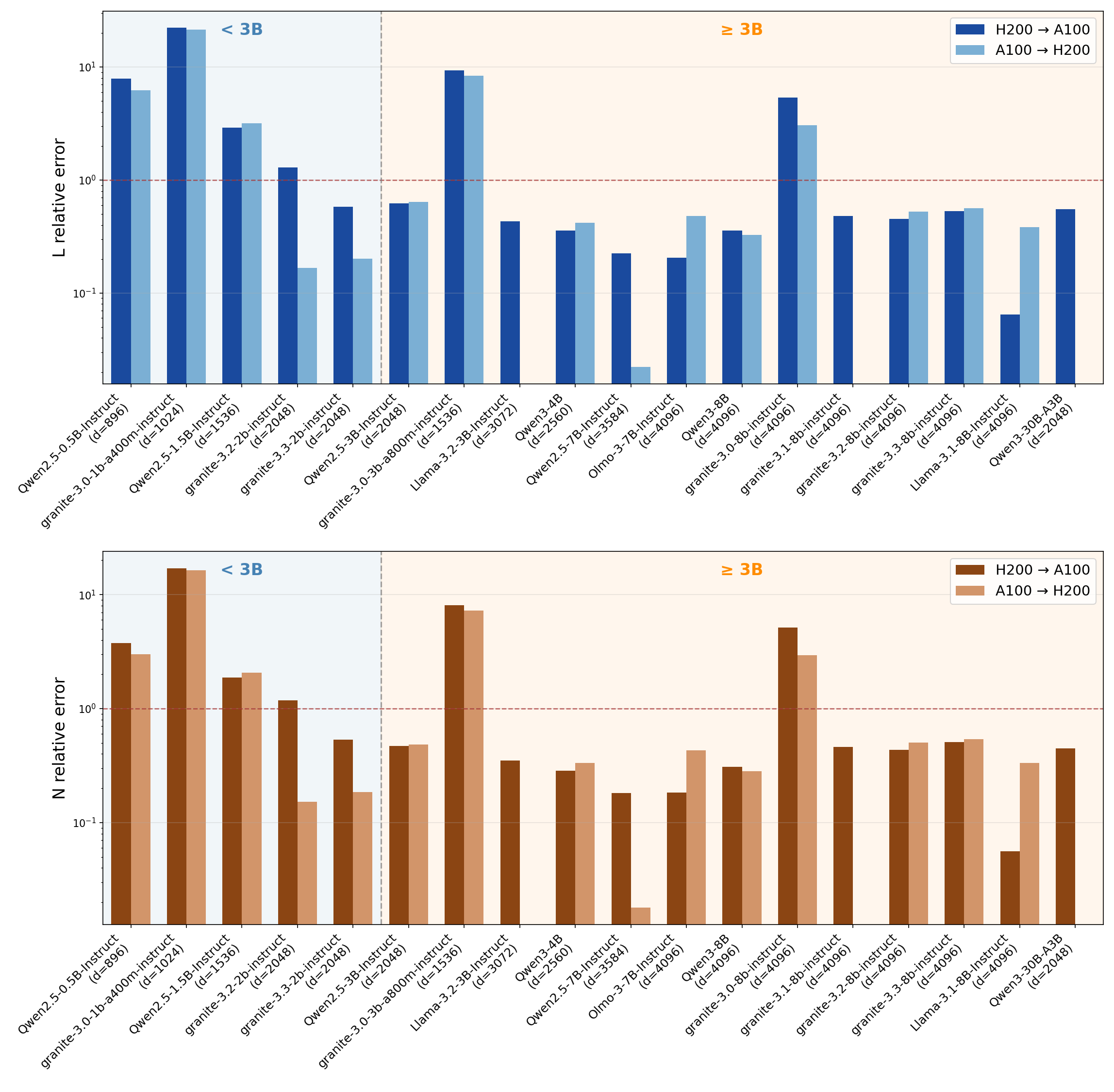}
	\caption{\small Cross-platform relative error for the estimated layer count (top) and total parameter count (bottom) across models, where the platform transfer is between NVIDIA H200 and A100 GPUs. Dark bars denote H200 as the reference platform and A100 as the target; light bars denote the reverse configuration. }
	\label{fig:cross_gpu}
\end{figure}

On the target platform, we assume a set of models available $\mathcal{F}^{\mathrm{tgt}} = \{\model^{\mathrm{tgt}}_i\}_{i=1}^{K'}$. 

Let $f^{\mathrm{tgt}}_g$ denote the target model, 
we fit a global relation using
spectral estimates $\dhat$ (Section~\ref{sec:spectral}):
\begin{equation}
	T_{i,j}^{\mathrm{tgt}} \approx \gamma^{\mathrm{tgt}}\,
	\ell_{j}^{2}\, \dhat_{i} + \zeta^{\mathrm{tgt}}.
	\label{eq:per_model_tgt}
\end{equation}
Assuming $\dhat \approx d$, we compute the platform transfer factor
\begin{equation}
\transferk = \gamma^{\mathrm{tgt}} / \gamma^{\mathrm{ref}} \; , \label{eq:transfer_factor}
\end{equation} which captures the
aggregate throughput ratio between the two platforms.  Applying $\transferk$ to the
layer-aware reference relation yields the cross-platform approximation
\begin{equation}
	T_{i,j}^{\mathrm{tgt}} \approx \transferk\, \beta^{\mathrm{ref}}\, L_i\,
	\ell_{j}^{2}\, d_{i} + \alpha^{\mathrm{ref}},
	\label{eq:cross_platform}
\end{equation}
from which the inferred layer count is recovered as:
\begin{equation}
	\Lhat_g = \frac{\gamma_{g}^{\mathrm{tgt}}}
	{\transferk \cdot \beta^{\mathrm{ref}}} = {\gamma_{g}^{\mathrm{tgt}} \cdot \gamma^{\mathrm{ref}} \over \gamma^{\mathrm{tgt}} \cdot \beta^{\mathrm{ref}}},
	\label{eq:L_inference_global}
\end{equation}
where $\gamma_{g}^{\mathrm{tgt}}$ is the held-out target model fitting parameter, obtained by solving \eqref{eq:per_model_tgt} for $\gamma^{\mathrm{tgt}}$.

This formulation assumes that the relative scaling behavior across layers is
preserved across platforms, i.e., that platform differences manifest primarily as a multiplicative constant rather than altering the functional form of the relation. 

Figure~\ref{fig:cross_gpu} presents cross-platform estimation results between NVIDIA H200 and A100 GPUs, showing the inferred layer count (top panel) and total parameter count (bottom panel). Compared with the
same-platform setting of Figure~\ref{fig:val_plot_n_layers} (reference and target models both on NVIDIA H200), the cross-platform results exhibit a similar range of relative errors and comparable trends across models, with
an average relative error of approximately $50\%$ and the same model-size dependency observed in the same-platform case (larger models tend to yield lower error). These results demonstrate that our cross-platform calibration procedure (Section~\ref{app:cross-hardware}) effectively compensates for
hardware-specific throughput differences without degrading inference accuracy.

\clearpage 
\section{Model List}
\label{app:models}

Table~\ref{tab:model-architectures} summarizes the architectural details of all
language models evaluated in this work. We report the configurations published on the
Hugging Face Hub. All values are extracted from each model's \texttt{config.json}.

The 39 models span eight families: Alibaba's Qwen 2.5, Qwen 3, and Qwen 3.5 series; IBM's Granite 3.0 through 3.3 and the Granite 4.0 dense models; Meta's Llama 3.1 and 3.2 instruction models; AI2's OLMo 2, OLMoE, and OLMo 3; and Hugging Face's SmolLM2 and SmolLM3. Together they cover more than two orders of magnitude of total
parameters (135M to 30B) and span various architectures in current use:
standard dense transformers (Qwen 2.5/3, Llama, OLMo 2, SmolLM, Granite 3.x and 4.0 dense),
sparse mixture-of-experts (Qwen 3 30B-A3B, Granite 3.x \moe, OLMoE), dense transformers with interleaved sliding-window and full softmax-attention layers (OLMo 3), and dense transformers with interleaved Gated DeltaNet linear-attention and full softmax-attention layers (Qwen 3.5).

The table reports the number of decoder layers $L$, the hidden dimension $d$, and the feed-forward intermediate dimension $\dffn$, which corresponds to the size of a single expert's MLP for mixture-of-experts (\moe) models and to the dense MLP
otherwise. Attention head counts are given as $h_q/h_{kv}$, where equal values
indicate standard multi-head attention and $h_{kv}<h_q$ indicates grouped-query
attention. The vocabulary size is denoted by $|\vocab|$ and the ``Tied'' column indicates whether the input embedding matrix is shared with the output projection. The ``Type'' column reports the type of attention if full softmax-attention is not used in all layers (full softmax-attention interleaved with sliding-window attention iSWA) or interleaved with Gated DeltaNet (iGDN) linear-attention (GDN)), and the type of feed-forward network (Dense or \moe). For \moe models we additionally list the total and active expert counts per layer, $\Etot/\Eact$, where $\Eact$ is the top-$k$ routing parameter. For \moe models we further report active parameters per
token, computed analytically from the configuration as
$|\vocab|\cdot d + L\cdot(P_{\mathrm{attn}} + \Eact\cdot
P_{\mathrm{expert}} + P_{\mathrm{router}})$, where $P_{\mathrm{attn}}$ is the
per-layer attention parameter count, $P_{\mathrm{expert}} = 3 d \cdot
\dffn$ is the per-expert SwiGLU MLP parameter count, and
$P_{\mathrm{router}} = d \cdot \Etot$ is the per-layer router
parameter count.

\paragraph{Development models.} The following seven model configurations were used during method development (hyperparameter selection, algorithmic choices, and prompt design) and are held out from the validation set reported in \Cref{sec:eval}:
\begin{itemize}[leftmargin=*]
	\item \texttt{HuggingFaceTB/SmolLM2-360M-Instruct}
	\item \texttt{HuggingFaceTB/SmolLM2-1.7B-Instruct}
	\item \texttt{HuggingFaceTB/SmolLM3-3B}
	\item \texttt{Qwen/Qwen3.5-9B}
	\item \texttt{Qwen/Qwen3-0.6B}
	\item \texttt{ibm-granite/granite-3.1-1b-a400m-instruct}
	\item \texttt{ibm-granite/granite-3.1-2b-instruct}
\end{itemize}
Together with the 32 validation models, these account for the 39 model configurations referenced throughout the paper.

\paragraph{Reasoning mode.} For Qwen models that support a toggleable reasoning (``thinking'') mode (\textit{e.g.}, the Qwen 3 and Qwen 3.5 series), we disable reasoning mode for all experiments by passing \texttt{enable\_thinking=False} to the chat template. This ensures that output tokens are sampled directly from the model's next-token distribution rather than from a post-reasoning distribution, keeping the log-probability and timing measurements consistent with the non-reasoning models in our suite.\\

\definecolor{familyqwen25}{RGB}{255,243,224}
\definecolor{familyqwen3}{RGB}{255,224,236}
\definecolor{familyqwen35}{RGB}{243,224,255}
\definecolor{familygranite}{RGB}{224,242,255}
\definecolor{familygranite4}{RGB}{208,232,255}
\definecolor{familyllama}{RGB}{224,255,229}
\definecolor{familyolmo}{RGB}{255,237,213}
\definecolor{familysmol}{RGB}{255,224,224}
\definecolor{familygpt}{RGB}{232,232,232}

\begin{table}[htpb]
	\centering
	\caption{Architectural details of the language models evaluated in experiments.}
	\label{tab:model-architectures}
	\setlength{\tabcolsep}{3.0pt}
	\renewcommand{\arraystretch}{1.0}
	\scriptsize
	\begin{tabular}{@{}lccccccccr@{}}
		\toprule
		\textbf{Model} & \textbf{$L$} & \textbf{$d$} & \textbf{$\dffn$} & \textbf{$h_q/h_{kv}$} & \textbf{$|\vocab|$} & \textbf{Tied} & \textbf{Type} & \textbf{$\Etot/\Eact$} & \textbf{Params (B)} \\
		\midrule
		
		\rowcolor{familyqwen25}[\tabcolsep][\tabcolsep]\multicolumn{10}{@{}l}{\textit{Qwen 2.5 (\texttt{Qwen/Qwen2.5-*-Instruct})}} \\
		\rowcolor{familyqwen25}[\tabcolsep][\tabcolsep] \quad 0.5B  & 24 &   896 &  4{,}864 & 14/2 & 151{,}936 & \checkmark & dense & --      & 0.49 \\
		\rowcolor{familyqwen25}[\tabcolsep][\tabcolsep] \quad 1.5B  & 28 & 1{,}536 &  8{,}960 & 12/2 & 151{,}936 & \checkmark & dense & --      & 1.54 \\
		\rowcolor{familyqwen25}[\tabcolsep][\tabcolsep] \quad 3B    & 36 & 2{,}048 & 11{,}008 & 16/2 & 151{,}936 & \checkmark & dense & --      & 3.09 \\
		\rowcolor{familyqwen25}[\tabcolsep][\tabcolsep] \quad 7B    & 28 & 3{,}584 & 18{,}944 & 28/4 & 152{,}064 & --         & dense & --      & 7.62 \\
		
		\addlinespace[1pt]
		\rowcolor{familyqwen3}[\tabcolsep][\tabcolsep]\multicolumn{10}{@{}l}{\textit{Qwen 3 (\texttt{Qwen/Qwen3-*})}} \\
		\rowcolor{familyqwen3}[\tabcolsep][\tabcolsep] \quad 0.6B            & 28 & 1{,}024 &  3{,}072 & 16/8 & 151{,}936 & \checkmark & dense & --     & 0.60 \\
		\rowcolor{familyqwen3}[\tabcolsep][\tabcolsep] \quad 1.7B            & 28 & 2{,}048 &  6{,}144 & 16/8 & 151{,}936 & \checkmark & dense & --     & 1.72 \\
		\rowcolor{familyqwen3}[\tabcolsep][\tabcolsep] \quad 4B              & 36 & 2{,}560 &  9{,}728 & 32/8 & 151{,}936 & \checkmark & dense & --     & 4.02 \\
		\rowcolor{familyqwen3}[\tabcolsep][\tabcolsep] \quad 8B              & 36 & 4{,}096 & 12{,}288 & 32/8 & 151{,}936 & --         & dense & --     & 8.19 \\
		\rowcolor{familyqwen3}[\tabcolsep][\tabcolsep] \quad 30B-A3B         & 48 & 2{,}048 &     768  & 32/4 & 151{,}936 & --         & MoE   & 128/8  & 30.5\,/\,3.4 \\
		\rowcolor{familyqwen3}[\tabcolsep][\tabcolsep] \quad 30B-A3B-Inst.   & 48 & 2{,}048 &     768  & 32/4 & 151{,}936 & --         & MoE   & 128/8  & 30.5\,/\,3.4 \\
		
		\addlinespace[1pt]
		\rowcolor{familyqwen35}[\tabcolsep][\tabcolsep]\multicolumn{10}{@{}l}{\textit{Qwen 3.5 (\texttt{Qwen/Qwen3.5-*};}} \\
		\rowcolor{familyqwen35}[\tabcolsep][\tabcolsep] \quad 0.8B    & 24 & 1{,}024 &  3{,}584 &  8/2 & 248{,}320 & \checkmark & dense + GDN     & --      & 0.87 \\
		\rowcolor{familyqwen35}[\tabcolsep][\tabcolsep] \quad 2B      & 24 & 2{,}048 &  6{,}144 &  8/2 & 248{,}320 & \checkmark & dense + GDN     & --      & 2.27 \\
		\rowcolor{familyqwen35}[\tabcolsep][\tabcolsep] \quad 4B      & 32 & 2{,}560 &  9{,}216 & 16/4 & 248{,}320 & \checkmark & dense + GDN     & --      & 4.66 \\
		\rowcolor{familyqwen35}[\tabcolsep][\tabcolsep] \quad 9B      & 32 & 4{,}096 & 12{,}288 & 16/4 & 248{,}320 & --         & dense + GDN     & --      & 9.65 \\
		%\rowcolor{familyqwen35}[\tabcolsep][\tabcolsep] \quad 35B-A3B & 40 & 2{,}048 &     512  & 16/2 & 248{,}320 & --         & MoE+GDN & 256/8  & 35.95\,/\,2.93 \\
		
		\addlinespace[1pt]
		\rowcolor{familygranite}[\tabcolsep][\tabcolsep]\multicolumn{10}{@{}l}{\textit{IBM Granite 3.0--3.3 (\texttt{ibm-granite/granite-3.x-*-instruct})}} \\
		\rowcolor{familygranite}[\tabcolsep][\tabcolsep] \quad 3.0-1B-A400M  & 24 & 1{,}024 &    512 & 16/8 & 49{,}155 & \checkmark & MoE   & 32/8  & 1.33\,/\,0.43 \\
		\rowcolor{familygranite}[\tabcolsep][\tabcolsep] \quad 3.0-3B-A800M  & 32 & 1{,}536 &    512 & 24/8 & 49{,}155 & \checkmark & MoE   & 40/8  & 3.37\,/\,0.88 \\
		\rowcolor{familygranite}[\tabcolsep][\tabcolsep] \quad 3.0-2B        & 40 & 2{,}048 &  8{,}192 & 32/8 & 49{,}155 & \checkmark & dense & --    & 2.63 \\
		\rowcolor{familygranite}[\tabcolsep][\tabcolsep] \quad 3.0-8B        & 40 & 4{,}096 & 12{,}800 & 32/8 & 49{,}155 & \checkmark & dense & --    & 8.17 \\
		\rowcolor{familygranite}[\tabcolsep][\tabcolsep] \quad 3.1-1B-A400M  & 24 & 1{,}024 &    512 & 16/8 & 49{,}155 & \checkmark & MoE   & 32/8  & 1.33\,/\,0.43 \\
		\rowcolor{familygranite}[\tabcolsep][\tabcolsep] \quad 3.1-3B-A800M  & 32 & 1{,}536 &    512 & 24/8 & 49{,}155 & \checkmark & MoE   & 40/8  & 3.30\,/\,0.88 \\
		\rowcolor{familygranite}[\tabcolsep][\tabcolsep] \quad 3.1-2B        & 40 & 2{,}048 &  8{,}192 & 32/8 & 49{,}155 & \checkmark & dense & --    & 2.53 \\
		\rowcolor{familygranite}[\tabcolsep][\tabcolsep] \quad 3.1-8B        & 40 & 4{,}096 & 12{,}800 & 32/8 & 49{,}155 & \checkmark & dense & --    & 8.17 \\
		\rowcolor{familygranite}[\tabcolsep][\tabcolsep] \quad 3.2-2B        & 40 & 2{,}048 &  8{,}192 & 32/8 & 49{,}155 & \checkmark & dense & --    & 2.53 \\
		\rowcolor{familygranite}[\tabcolsep][\tabcolsep] \quad 3.2-8B        & 40 & 4{,}096 & 12{,}800 & 32/8 & 49{,}155 & \checkmark & dense & --    & 8.17 \\
		\rowcolor{familygranite}[\tabcolsep][\tabcolsep] \quad 3.3-2B        & 40 & 2{,}048 &  8{,}192 & 32/8 & 49{,}159 & \checkmark & dense & --    & 2.53 \\
		\rowcolor{familygranite}[\tabcolsep][\tabcolsep] \quad 3.3-8B        & 40 & 4{,}096 & 12{,}800 & 32/8 & 49{,}159 & \checkmark & dense & --    & 8.17 \\
		
		\addlinespace[1pt]
		\rowcolor{familygranite4}[\tabcolsep][\tabcolsep]\multicolumn{10}{@{}l}{\textit{IBM Granite 4.0 (\texttt{ibm-granite/granite-4.0-*}; }} \\
		\rowcolor{familygranite4}[\tabcolsep][\tabcolsep] \quad 350M  & 28 & 1{,}024 & 2{,}048 & 16/4 & 100{,}352 & \checkmark & hybrid & -- & 0.35 \\
		\rowcolor{familygranite4}[\tabcolsep][\tabcolsep] \quad Micro & 40 & 2{,}560 & 8{,}192 & 40/8 & 100{,}352 & \checkmark & hybrid & -- & 3.40 \\
		
		\addlinespace[1pt]
		\rowcolor{familyllama}[\tabcolsep][\tabcolsep]\multicolumn{10}{@{}l}{\textit{Meta Llama 3.x (\texttt{meta-llama/Llama-3.x-*-Instruct})}} \\
		\rowcolor{familyllama}[\tabcolsep][\tabcolsep] \quad 3.2-1B  & 16 & 2{,}048 &  8{,}192 & 32/8 & 128{,}256 & \checkmark & dense & -- &  1.24 \\
		\rowcolor{familyllama}[\tabcolsep][\tabcolsep] \quad 3.2-3B  & 28 & 3{,}072 &  8{,}192 & 24/8 & 128{,}256 & \checkmark & dense & -- &  3.21 \\
		\rowcolor{familyllama}[\tabcolsep][\tabcolsep] \quad 3.1-8B  & 32 & 4{,}096 & 14{,}336 & 32/8 & 128{,}256 & --         & dense & -- &  8.03 \\

		\addlinespace[1pt]
		\rowcolor{familyolmo}[\tabcolsep][\tabcolsep]\multicolumn{10}{@{}l}{\textit{AI2 OLMo (\texttt{allenai/OLMo*-*-Instruct})}} \\
		\rowcolor{familyolmo}[\tabcolsep][\tabcolsep] \quad OLMo-2-1B (0425)    & 16 & 2{,}048 &  8{,}192 & 16/16 & 100{,}352 & -- & dense & --    & 1.48 \\
		\rowcolor{familyolmo}[\tabcolsep][\tabcolsep] \quad OLMo-2-7B (1124)    & 32 & 4{,}096 & 11{,}008 & 32/32 & 100{,}352 & -- & dense & --    & 7.30 \\
		\rowcolor{familyolmo}[\tabcolsep][\tabcolsep] \quad OLMoE-1B-7B (0924)  & 16 & 2{,}048 &  1{,}024 & 16/16 &  50{,}304 & -- & MoE   & 64/8  & 6.92\,/\,1.28 \\
		\rowcolor{familyolmo}[\tabcolsep][\tabcolsep] \quad OLMo-3-7B           & 32 & 4{,}096 & 11{,}008 & 32/32 & 100{,}278 & -- & dense + SWA & --    & 7.30 \\
		
		\addlinespace[1pt]
		\rowcolor{familysmol}[\tabcolsep][\tabcolsep]\multicolumn{10}{@{}l}{\textit{HuggingFace SmolLM (\texttt{HuggingFaceTB/SmolLM*-*})}} \\
		\rowcolor{familysmol}[\tabcolsep][\tabcolsep] \quad SmolLM2-135M  & 30 &   576 &  1{,}536 &  9/3   &  49{,}152 & \checkmark & dense & -- & 0.13 \\
		\rowcolor{familysmol}[\tabcolsep][\tabcolsep] \quad SmolLM2-360M  & 32 &   960 &  2{,}560 & 15/5   &  49{,}152 & \checkmark & dense & -- & 0.36 \\
		\rowcolor{familysmol}[\tabcolsep][\tabcolsep] \quad SmolLM2-1.7B  & 24 & 2{,}048 &  8{,}192 & 32/32 &  49{,}152 & \checkmark & dense & -- & 1.71 \\
		\rowcolor{familysmol}[\tabcolsep][\tabcolsep] \quad SmolLM3-3B    & 36 & 2{,}048 & 11{,}008 & 16/4  & 128{,}256 & \checkmark & dense & -- & 3.08 \\

		\bottomrule
	\end{tabular}
\end{table}

\clearpage
\section{Experimental Setup Details}
\label{app:timing-experimental-setup}
\paragraph{Hardware}
Each model runs on a single NVIDIA H100 (RHEL~8 / glibc 2.28; CUDA~12.6; PyTorch with CUDA~12.6; Python~3.12). Timing measurements are additionally repeated on H200 and A100 GPUs to test cross-platform calibration.

\textbf{Timing protocol.} For each model, we build synthetic prompts of $\ell$ tokens by repeating a fixed passage. After three warmup requests, we issue $100$ batch-size-one chat requests generating one output token, measuring client-side latency via \texttt{time.time()} and subtracting an $\ell{=}1$ baseline to isolate the prefill cost. At each $\ell$ we report the $10\%$-trimmed mean over $N{=}100$ trials, weighting regressions by $\sigma_T = \mathrm{std}/\sqrt{N_{\text{kept}}}$. We sweep $n=24$ lengths spanning the projection-dominated ($O(L\ell d^2)$) to attention-dominated ($O(L\ell^2 d)$) transition, with crossover $\ell^{\star} = \lfloor(d + \dkv + \tfrac{3}{2} \dffn) r \rfloor$, where $\dkv = \nkv(d/\nheads)$ and $r \approx 0.6$ is the empirical attention-to-MLP throughput ratio; points span $[0.5\ell^{\star},\, \min(3\ell^{\star}, \ell_{\mathrm{cap}})]$ with $\ell_{\mathrm{cap}} = \min(\max(2\ell^{\star}, 4d), p_{\max}-200)$, and we discard 12 points to restrict the fit to the attention-dominated regime.

\paragraph{Sequence-length sweep}
We sweep input lengths $\mathcal{L} = \{\ell_j\}_{j=1}^{N'}$ chosen to span the transition between the projection-dominated ($O(L\,\ell\,d^2)$) and attention-dominated ($O(L\,\ell^2\,d)$) regimes. We estimate the crossover as
\begin{equation}
    \ell^{\star} = \left\lfloor\left(d + \dkv
    + \tfrac{3}{2}\, \dffn\right) \cdot r \right\rfloor,
    \label{eq:crossover}
\end{equation}
where $\dkv = \nkv \cdot (d / \nheads)$ with $\nheads$ attention heads and $\nkv$ KV heads, $\dffn$ is the feed-forward intermediate dimension, and $r \approx 0.6$ reflects the empirical attention-to-MLP throughput ratio. We sweep $n = 24$ uniformly spaced points from $\ell_{\mathrm{lo}} = 0.5\, \ell^{\star}$ to $\ell_{\mathrm{hi}} = \min(3\, \ell^{\star},\, \ell_{\mathrm{cap}})$, where $\ell_{\mathrm{cap}} = \min(\max(2\, \ell^{\star},\, 4d),\; p_{\max} - 200)$ keeps sequences within the positional encoding limit $p_{\max}$, and additionally include $\ell^{\star}$ when it falls inside the sweep. When fitting the scaling relation, we discard 12 points to restrict the fit to the attention-dominated regime, excluding sequences too short for prefill cost to dominate as well as outliers.

\clearpage
\section{Additional Results}
\label{app:additional_results}

\subsection{Spectral Shape and Eigenspectrum Saturation}
\label{app:spectral-shape}
\begin{figure}[htpb]
	\centering
	\includegraphics[width=11cm]{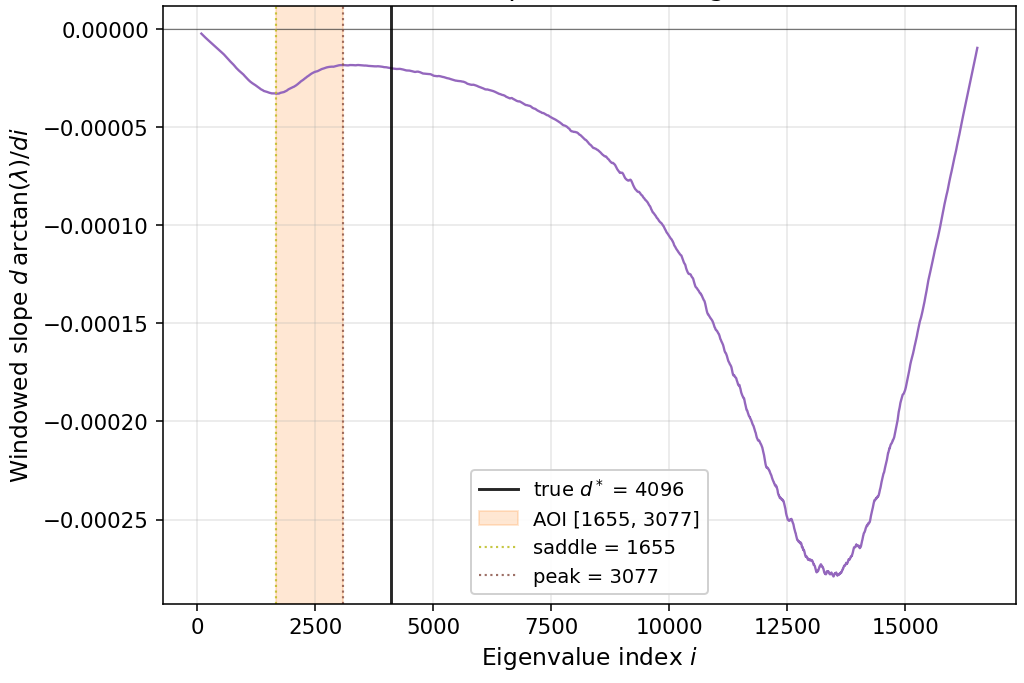}
	\caption{Windowed slope $d\arctan(\lambda)/di$ of the ordered eigenvalues for Llama-3.1-8B at the maximum sample budget ($N{=}10^{6}$ per prompt). The true hidden dimension $d^{*}{=}4{,}096$ is marked with a vertical line. The change-point detector (\Cref{app:changepoint}) localizes the elbow inside the area of interest (\aoi) bracketed by the saddle (1655) and peak (3077). Both bracket landmarks fall well below the true $d^{*}$, so the estimate $\dhat$ underestimates the rank: at this budget the eigenspectrum has not yet saturated, and the elbow associated with the true rank is not yet pronounced enough for the detector to recover it. Larger sample budgets are needed to push the \aoi up to $d^{*}$.}
	\label{fig:llama_8b_shape}
\end{figure}
\newpage
\subsection{Timing}
\label{app:additional_timing}
\begin{figure}[htpb]
	\centering
	\includegraphics[width=14cm]{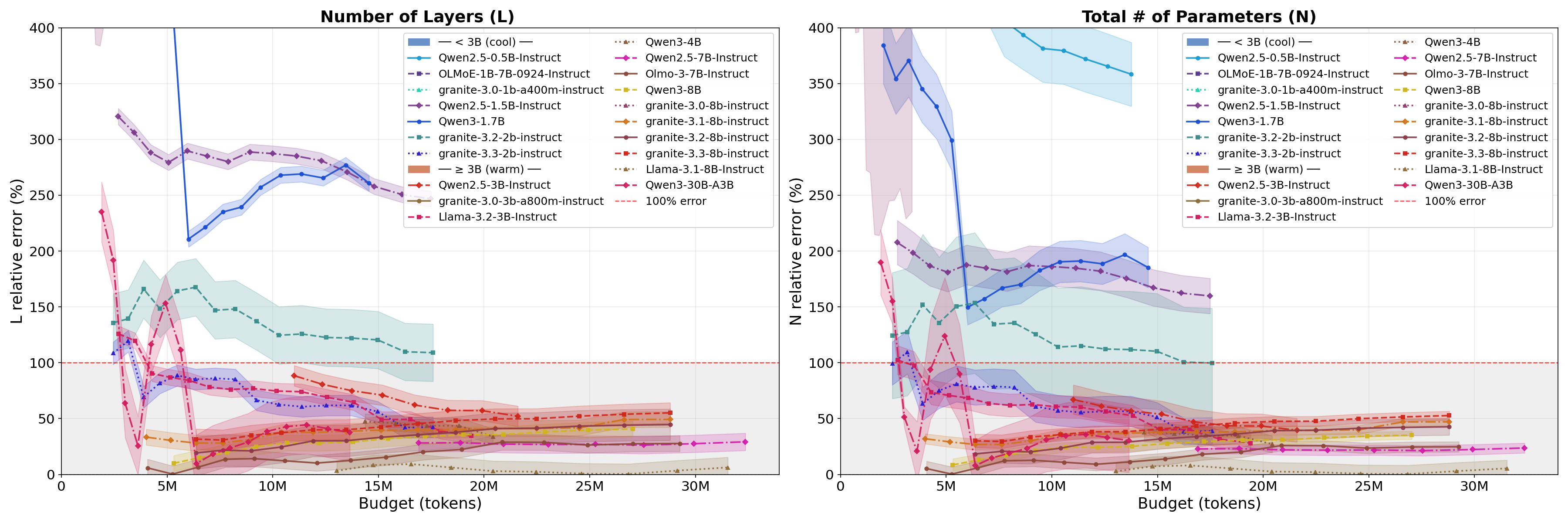}
	\caption{\small Relative error in estimating the number of layers and total number of parameters as a function of API query budget (measured in number of tokens). Some of the largest errors are driven by small models with $\leq$ 4B parameters.
	}
	\label{fig:L_param_err_budget_per_model}
\end{figure}

\begin{figure}[htpb]
	\centering
	\includegraphics[width=14cm]{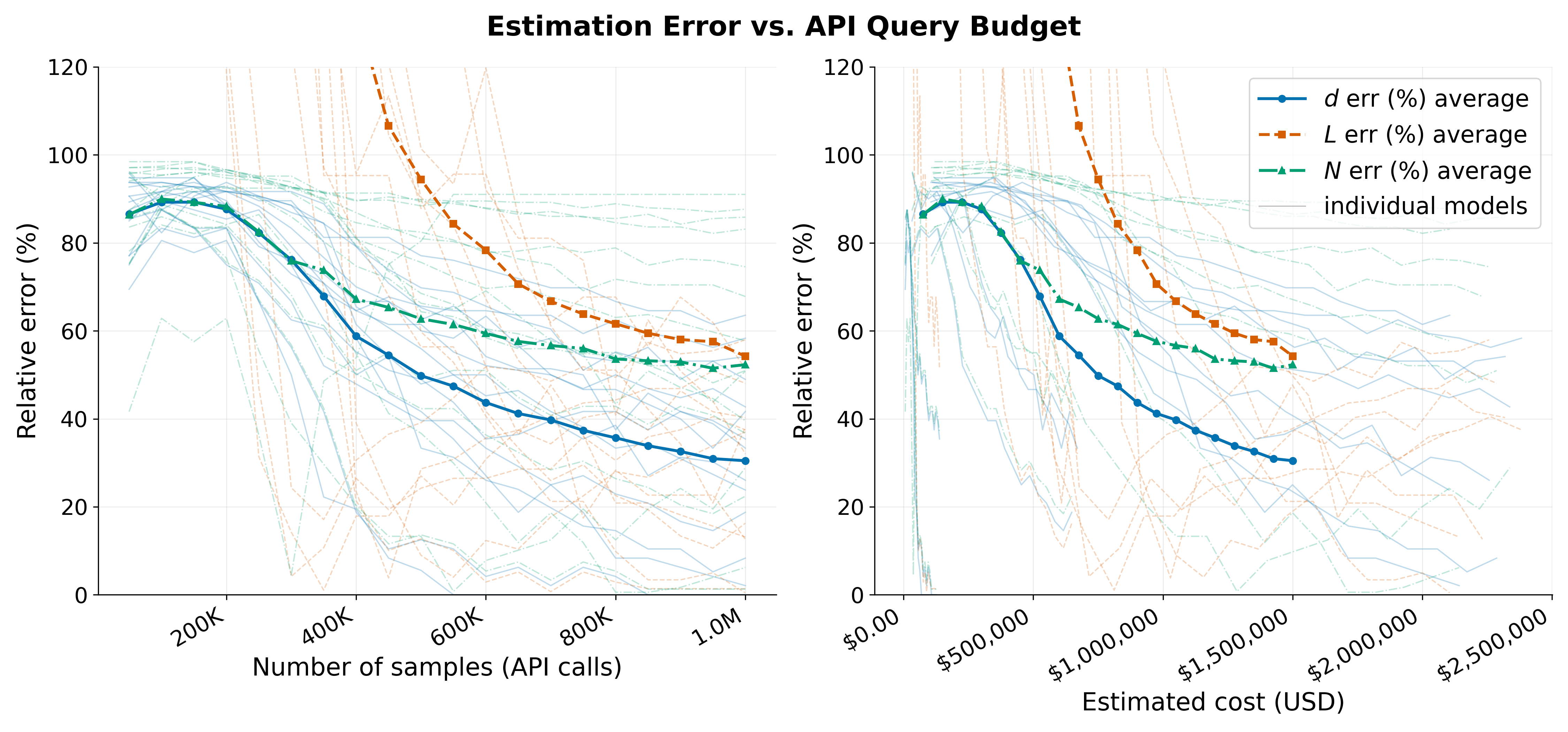}
	\caption{\small Estimation relative error versus number of API calls (left panel) and estimated cost in USD (right panel). Relative errors for d, L, and N are shown as a function of total token budget, with bold curves denoting averages across target models and faint curves denoting individual models. Larger budgets generally reduce error, with L improving fastest and N remaining harder to estimate. Results include $\geq$3B-parameter dense models and \moe models.
	}
	\label{fig:L_param_err_budget}
\end{figure}

\clearpage
\section{Proofs}
\label{app:proofs}

Here we prove \Cref{thm:uniform-sample-complexity} in \Cref{sec:theory} and provide
a generalization to the non-uniform case. Throughout, we let $\vocab$ be a finite
vocabulary of size $\vocabsize$, with tokens $v_1, \dots, v_{\vocabsize}$. We draw
$\inputprompts$ random subsets
$\tokenset_1, \tokenset_2, \dots, \tokenset_{\inputprompts} \subseteq \vocab$ as
follows: for each $q \in [\inputprompts]$ we sample $N$ tokens from $\vocab$
independently and with replacement, and let $\tokenset_q$ be the set of distinct
tokens obtained. The $q$-th subset is in general drawn according to its own
probability distribution over $\vocab$,
\begin{equation}
  \pvec_q = \left(p^{(1)}_q, p^{(2)}_q, \dots, p^{(\vocabsize)}_q\right),
  \qquad \sum_{i=1}^{\vocabsize} p^{(i)}_q = 1 .
\end{equation}
The subscript $q$
indexes the subset (which of the $\inputprompts$ distributions is in force) while the
parenthesized superscript indexes the token; hence $p^{(i)}_q$ is the
probability of drawing token $v_i$ when forming $\tokenset_q$. The uniform case of
\Cref{thm:uniform-sample-complexity} is recovered by taking every distribution to be
uniform over $\vocab$, i.e.\ $p^{(i)}_q = 1/\vocabsize$ for every token $v_i$,
$i \in [\vocabsize]$, and every subset $q \in [\inputprompts]$. We drop the subscript
$q$ whenever the distribution does not depend on $q$, writing
$\pvec = (p^{(1)}, \dots, p^{(\vocabsize)})$. Finally, we denote by
$\commonset \triangleq \bigcap_{q \in [\inputprompts]} \tokenset_q$ the common
intersection of the $\inputprompts$ subsets, i.e., the tokens that appear in every
subset.

\begin{lemma}[Negative association of intersection indicators]
\label{lem:na-indicators}
Let $X_i \triangleq \mathbf{1}(v_i \in \commonset) \in \{0,1\}$ be the indicator random variable for the event that token $v_i$ is in the common token set $\commonset$.
Then $X_1, \dots, X_V$ are negatively associated.
\end{lemma}
\begin{proof}
A collection of random variables $(X_1, \dots, X_m)$ is \emph{negatively associated}
(NA) if for every pair of disjoint index sets $I, J \subseteq [m]$ and every pair of
functions $f, g$ that are both non-decreasing (or both non-increasing),
\[
    \E\left[f(X_i, i \in I)\, g(X_j, j \in J)\right]
    \;\leq\; \E\left[f(X_i, i \in I)\right]\,\E\left[g(X_j, j \in J)\right].
\]

NA captures the idea that the variables compete: when one subset of
them takes on larger values, any disjoint subset tends to take on smaller values, so
they cannot all be large simultaneously. NA is stronger than negative correlation. Correlation tracks how variables move in pairs, whereas NA says that any group of them competes with any disjoint group. So NA variables are always negatively correlated, but variables can be negatively correlated in pairs without exhibiting this fuller competition. We refer the reader to~\citet{dubhashi1996balls} for additional background.

In our setting, we perform $N$ draws in each of the $\inputprompts$ subsets, thus $N\inputprompts$ draws in total. For each $q \in [\inputprompts]$ and
$k \in [N]$, let $\mathbf{w}_{q,k} \in \{0,1\}^{\vocabsize}$ be the one-hot encoding
of the $k$-th token drawn in subset $q$. Its coordinate $w_{q,k}^{(i)}$ indicates
whether token $v_i$ was the one drawn on the $k$-th draw of subset $q$. By
construction $\sum_{i=1}^{\vocabsize} w_{q,k}^{(i)} = 1$ with each
$w_{q,k}^{(i)} \in \{0,1\}$, so by Lemma~8 of~\citet{dubhashi1996balls} the entries of
$\mathbf{w}_{q,k}$ are NA. 

Next, we show that NA of $\mathbf{w}_{q,k}$ implies NA of the indicators $(X_1, \dots, X_{\vocabsize})$, where
$X_i \triangleq \mathbf{1}(v_i \in \commonset)$ records whether token $v_i$ is drawn
at least once in every one of the $\inputprompts$ subsets. The indicators can be
expressed in terms of the one-hot encodings as
\begin{align}
    X_i &= \prod_{q=1}^{\inputprompts} \max_{k \in [N]} w_{q,k}^{(i)},
\end{align}
since $\max_{k \in [N]} w_{q,k}^{(i)} = 1$ exactly when $v_i \in \tokenset_q$, and the
product over $q$ is $1$ exactly when $v_i \in \tokenset_q$ for all $q$, i.e.\
$v_i \in \commonset$.

Each $\mathbf{w}_{q,k}$ is a deterministic function of a single draw, and the
$N\inputprompts$ draws are mutually independent, so the vectors
$\{\mathbf{w}_{q,k}\}$ are mutually independent. Combining the within-draw NA of each $\mathbf{w}_{q,k}$ with this independence, the union of
independent sets of NA random variables is again NA by property~7
of~\citet{JoagProschan83}, so the full collection
$\{w_{q,k}^{(i)}\}_{i,q,k}$ is NA. Each indicator $X_i$ depends only on the
coordinates $\{w_{q,k}^{(i)} : q \in [\inputprompts], k \in [N]\}$ indexed by token
$v_i$, and these subsets are disjoint across distinct $i$. Since $\max$ and products
of $\{0,1\}$-valued terms are non-decreasing, each $X_i$ is a non-decreasing function
defined over disjoint subsets of NA random variables, and hence the indicators are NA
by property~6 of~\citet{JoagProschan83}.
\end{proof}

\SampleComplexity*
\begin{proof}
	The intersection size can be expressed as
	$C \triangleq |\commonset| = \sum_{i=1}^{\vocabsize} X_i$, where
	$X_i = \mathbf{1}(v_i \in \commonset)$ as in \Cref{lem:na-indicators}.

	\emph{Step 1: Expected intersection size.} By independence across the
	$\inputprompts$ subsets,
	\begin{align}
		\mu \triangleq \E[C]
		= \sum_{i=1}^{\vocabsize} \E[X_i]
		= \sum_{i=1}^{\vocabsize} \prod_{q=1}^{\inputprompts}
		  \P(v_i \in \tokenset_q).
	\end{align}
	Since each draw within $\tokenset_q$ samples token $v_i$ uniformly with
	probability $1/\vocabsize$, the probability that $v_i$ never appears in $N$
	draws is $(1 - 1/\vocabsize)^N$, so
	$\P(v_i \in \tokenset_q) = 1 - (1 - 1/\vocabsize)^N$. Substituting gives
	\begin{align}
		\mu = \vocabsize\bigl(1 - (1 - 1/\vocabsize)^N\bigr)^{\inputprompts}.
		\label{eq:mu-uniform}
	\end{align}

	\emph{Step 2: Concentration via a Chernoff bound for NA variables.} By~\Cref{lem:na-indicators}, $C$ is a sum of the NA indicators $X_1, \dots, X_{\vocabsize} \in \{0,1\}$. The multiplicative Chernoff lower-tail bound holds exactly as in the
	independent case~\citep{dubhashi1996balls}. Thus for any $\varepsilon \in (0,1)$,
	\begin{align}
		\P\bigl(C \leq (1 - \varepsilon)\mu\bigr)
		\leq \exp\!\left(-\frac{\varepsilon^2 \mu}{2}\right).
		\label{eq:chernoff}
	\end{align}
	Setting the right-hand side equal to $\delta$ gives
	$\varepsilon = \sqrt{2\ln(1/\delta)/\mu}$, so with probability at least
	$1 - \delta$,
	\begin{align}
		C \geq \mu - \sqrt{2\mu \ln(1/\delta)}.
		\label{eq:lower-tail}
	\end{align}

	\emph{Step 3: Sufficient condition on $\mu$.} It therefore suffices to ensure
	$\mu - \sqrt{2\mu \ln(1/\delta)} \geq d$. Writing $u = \sqrt{\mu}$ and
	$a = \sqrt{2\ln(1/\delta)}$, this is the quadratic inequality
	$u^2 - a u - d \geq 0$, whose positive root gives
	$u \geq \tfrac{1}{2}(a + \sqrt{a^2 + 4d})$. Squaring and using
	$\sqrt{a^2 + 4d} \leq a + 2\sqrt{d}$,
	\begin{align}
		\mu \geq d + \sqrt{2d\ln(1/\delta)} + 2\ln(1/\delta).
		\label{eq:mu-condition}
	\end{align}

	\emph{Step 4: Solving for $N$.} Combining \eqref{eq:mu-uniform} and
	\eqref{eq:mu-condition}, it suffices that
	\begin{align}
		\vocabsize\bigl(1 - (1 - 1/\vocabsize)^N\bigr)^{\inputprompts}
		\geq d + \sqrt{2d\ln(1/\delta)} + 2\ln(1/\delta).
	\end{align}
	Solving for $N$,
	\begin{align}
		N \geq \frac{\ln\!\left[1 - \left(\dfrac{d + \sqrt{2d\ln(1/\delta)}
		+ 2\ln(1/\delta)}{\vocabsize}\right)^{1/\inputprompts}\right]}
		{\ln(1 - 1/\vocabsize)}.
	\end{align}
\end{proof}

\paragraph{Interpretation of the bound.} Here we interpret sample complexity provided by the bound in~\Cref{thm:uniform-sample-complexity}. Let $m$ be the required value of the required value of $\mu$ from \eqref{eq:mu-condition}, namely
\[
    m \;\triangleq\; d + \sqrt{2d\ln(1/\delta)} + 2\ln(1/\delta),
\] 
so that the sufficient
condition is $N \ge N^\star$ with
\[
    N^\star \;\triangleq\;
    \frac{-\ln\!\left[1 - (m/\vocabsize)^{1/\inputprompts}\right]}
         {-\ln(1 - 1/\vocabsize)}.
\]
We proceed bound the numerator and denominator. The denominator satisfies
$-\ln(1 - 1/\vocabsize) \ge 1/\vocabsize$, so its reciprocal is at most $\vocabsize$.
For the numerator, let $(m/\vocabsize)^{1/\inputprompts} = e^{-t}$ with
$t \triangleq \tfrac{1}{\inputprompts}\ln\tfrac{\vocabsize}{m}$. The inequality
$1 - e^{-t} \ge \tfrac{t}{1+t}$ gives
$-\ln(1 - e^{-t}) \le \ln(1 + 1/t)$, and in the regime $m \le \vocabsize/e$ we have
$t \ge 1/\inputprompts$, hence $1/t \le \inputprompts$. Combining,
\[
    N^\star \;\le\; \vocabsize \, \ln(1 + \inputprompts)
    \;=\; O(\vocabsize \ln \inputprompts),
\]
so it suffices to take $N = O(\vocabsize \ln \inputprompts)$ samples per subset, for a
total query budget
$\inputprompts N = O(\vocabsize\,\inputprompts \ln \inputprompts)$.

\Cref{thm:uniform-sample-complexity} assumes the output distribution over tokens is exactly uniform. In practice, the proposed common-set prompting search method finds prompts whose resulting output distributions are approximately, but not perfectly, uniform in the high temperature. We subsequently generalize the theorem by introducing a notion of $t$-flatness, which captures the degree to which an output distribution is spread across a sub-vocabulary of interest.
\begin{restatable}{definition}{tflat}[$t$-flat distribution]
	\label{def:t-flat}
	A probability distribution $\pvec = (p^{(1)}, \dots, p^{(\vocabsize)})$ over a
	vocabulary $\vocab$ is \emph{$t$-flat over $\vocab' \subseteq \vocab$} if
	$p^{(i)} \geq t$ for all $v_i \in \vocab'$.
\end{restatable}

Equipped with this definition, we generalize
\Cref{thm:uniform-sample-complexity} to the non-uniform setting where each of the $\inputprompts$ subsets may be drawn
from its own $\pvec_q$, provided they share a common floor $t$ on the sub-vocabulary $\vocab'$ of interest.

\begin{restatable}{corollary}{tflatsc}[Sample complexity under $t$-flat distributions]
	\label{cor:t-flat-sample-complexity}
	Let $\tokenset_1, \dots, \tokenset_{\inputprompts}$ be $\inputprompts$ independent
	random subsets of $\vocab$, each formed by drawing $N$ tokens with replacement from
	its own distribution $\pvec_q$, where every $\pvec_q$ is $t$-flat over a common
	sub-vocabulary $\vocab' \subseteq \vocab$. If
	\begin{equation}
		N \geq \frac{\ln\!\left[1 - \left(\frac{d + \sqrt{2d\ln(1/\delta)} + 2\ln(1/\delta)}{|\vocab'|}\right)^{1/\inputprompts}\right]}{\ln(1 - t)}
	\end{equation}
	then $|\commonset| \geq d$ with probability at least $1 - \delta$.
\end{restatable}
\begin{proof}
	Let $C' \triangleq \sum_{v_i \in \vocab'} X_i$ be the intersection size restricted
	to $\vocab'$, with $X_i = \mathbf{1}(v_i \in \commonset)$ as in
	\Cref{lem:na-indicators}. Since $\vocab' \subseteq \vocab$, $C' \leq |\commonset|$,
	so it suffices to guarantee $C' \geq d$. The indicators $\{X_i : v_i \in \vocab'\}$
	are a subcollection of the NA family of \Cref{lem:na-indicators} and any subset of an NA family is NA. Steps~2 and 3 of the proof of
	\Cref{thm:uniform-sample-complexity} therefore apply to $C'$ verbatim and it is sufficient that $\mu' \triangleq \E[C'] \geq d + \sqrt{2d\ln(1/\delta)} + 2\ln(1/\delta)$. By independence across the $\inputprompts$ subsets and $t$-flatness, $p_q^{(i)} \geq t$ for every
	$q \in [\inputprompts]$ and $v_i \in \vocab'$, so
	\begin{align}
		\mu' = \sum_{v_i \in \vocab'} \prod_{q=1}^{\inputprompts}
		       \bigl(1 - (1 - p_q^{(i)})^N\bigr)
		\;\geq\; |\vocab'|\bigl(1 - (1 - t)^N\bigr)^{\inputprompts}.
	\end{align}
	Ensuring the right-hand side is at least
	$d + \sqrt{2d\ln(1/\delta)} + 2\ln(1/\delta)$ and solving for $N$ (as in Step~4 of
	\Cref{thm:uniform-sample-complexity}, using $\ln(1-t) < 0$) gives the stated bound.
\end{proof}

%%%%%%%%%%%%%%%%%%%%%%%%%%%%%%%%%%%%%%%%%%%%%%%%%%%%%%%%%%%%

\end{document}